\documentclass[a4paper,11pt]{article}

\usepackage[latin1]{inputenc}
\usepackage[english]{babel}
\usepackage{graphicx}
\usepackage{amssymb}
\usepackage{amsmath}%
\usepackage{amsfonts}%
\usepackage{amsthm}
\usepackage{algorithm}
\usepackage{algorithmic}
\usepackage{url}
\usepackage{pdflscape}

\renewcommand{\algorithmicrequire}{\textbf{Input:}}
\renewcommand{\algorithmicensure}{\textbf{Output:}}

\newtheorem{theorem}{Theorem}

\newtheorem{corollary}{Corollary}

\newtheorem{lemma}{Lemma}

\newtheorem{remark}{Remark}

\numberwithin{equation}{section}
\DeclareMathOperator{\argmin}{argmin}
\DeclareMathOperator{\argmax}{argmax}

\DeclareMathOperator{\conv}{conv}

\DeclareMathOperator{\rank}{rank}

\title{ Hierarchical Clustering of Hyperspectral Images \\ using Rank-Two Nonnegative Matrix Factorization  }

\author{
Nicolas Gillis\thanks{Department of Mathematics and Operational Research, Facult\'e Polytechnique, Universit\'e de Mons, Rue de Houdain 9, B-7000 Mons, Email: nicolas.gillis@umons.ac.be. This work was carried on when  
NG was a postdoctoral researcher of the fonds de la recherche scientifique (F.R.S.-FNRS).} 
 \and Da Kuang\thanks{School of Computational Science and Engineering, Georgia Institute of Technology, Atlanta, GA 30332-0765, USA. Emails: $\{$da.kuang,hpark$\}$@cc.gatech.edu. The work of these authors was supported in part by the National Science Foundation (NSF) grants CCF-0808863 and CCF-0732318.} \and  Haesun Park$^{\dagger}$
} 

\date{}

\begin{document}
\renewcommand{\labelitemi}{$\diamond$}

\maketitle

\begin{abstract}

In this paper, we design a hierarchical clustering algorithm for high-resolution hyperspectral images. 
At the core of the algorithm, a new rank-two nonnegative matrix factorizations (NMF) algorithm is used to split the clusters, which is motivated by convex geometry concepts. The method starts with a single cluster containing all pixels, and, at each step, (i)~selects a cluster in such a way that the error at the next step is minimized, and (ii)~splits the selected cluster into two disjoint clusters using rank-two NMF in such a way that the clusters are well balanced and stable. The proposed method can also be used as an endmember extraction algorithm in the presence of pure pixels.  The effectiveness of this approach is illustrated on several synthetic and real-world hyperspectral images, and shown to outperform standard clustering techniques such as k-means, spherical k-means and standard NMF. 

\end{abstract}


\section{Introduction} \label{intro}

A hyperspectral image (HSI) is a set of images taken at many different wavelengths (usually between 100 and 200), not just the usual three visible bands of light (red at 650nm, green at 550nm, and blue at 450nm). An important problem in hyperspectral imaging is blind hyperspectral unmixing (blind~HU): given a HSI, the goal is to recover the constitutive materials present in the image (the \emph{endmembers})  and the corresponding abundance maps (that is, determine which pixel contains which endmember and in which quantity). Blind HU has many applications such as quality control in the food industry, analysis of the composition of chemical compositions and reactions, monitoring the development and health of crops, monitoring polluting sources, military surveillance, and medical imaging; see, e.g.,~\cite{Jose12} and the references therein. 

Let us associate a matrix $M \in \mathbb{R}^{m \times n}_+$ to a given HSI with $m$ spectral bands and $n$ pixels as follows: the $(i,j)$th entry $M(i,j)$ of matrix $M$ is the reflectance of the $j$th pixel at the $i$th wavelength (that is, the fraction of incident light that is reflected by the $i$th pixel at the $j$th wavelength). Hence each column of $M$ is equal to the spectral signature of a pixel 
while each row is a vectorized image at a given wavelength. 
The linear mixing model (LMM) assumes that the spectral signature of each pixel is a linear combination of the spectral signatures of the endmembers, where the weights in the linear combination are the abundances of each endmember in that pixel. For example, if a pixel contains 40\% of aluminum and 60\% of copper, then its spectral signature will be 0.4 times the spectral signature of the aluminum plus 0.6 times the spectral signature of the copper. This is a rather natural model: we assume that 40\% of the light is reflected by the aluminum while 60\% is by the copper, while non-linear effects are neglected (such as the light interacting with multiple materials before reflecting off, or atmospheric distortions). 

Assuming the image contains $r$ endmembers, and denoting $W(:,k) \in \mathbb{R}^m$ ($1 \leq k \leq r$) the spectral signatures of the endmembers, the LMM can be written as 
\[
M(:,j) = \sum_{k=1}^r W(:,k) H(k,j) \quad 1 \leq j \leq n, 
\]
where $H(k,j)$ is the abundance of the $k$th endmember in the $j$th pixel, hence $\sum_{k=1}^r H(k,j) = 1$ for all~$j$, which is referred to as the abundance sum-to-one constraint. 
Under the LMM and given a HSI~$M$, blind HU amounts to recovering the spectral signatures of the endmembers (matrix $W$) along with the abundances (matrix $H$). 
Since all matrices involved $M$, $W$ and $H$ are nonnegative, blind HU under the LMM is equivalent to nonnegative matrix factorization (NMF): Given a nonnegative matrix $M \in \mathbb{R}^{m \times n}_+$ and a factorization rank $r$, find two nonnegative matrices $W \in \mathbb{R}^{m \times r}_+$ and $H \in \mathbb{R}^{r \times n}_+$ such that $M \approx WH$. Unfortunately, NMF is NP-hard~\cite{V09} and highly ill-posed~\cite{G12}. 
Therefore, in practice, it is crucial to use the structure of the problem at hand to develop efficient numerical schemes for blind HU. This is usually achieved using additional constraints or regularization terms in the objective function, e.g., the sum-to-one constraint on the columns of $H$ (see above), sparsity of the abundance matrix $H$ (most pixels contain only a few endmembers), piecewise smoothness of the spectral signatures $W(:,k)$ ~\cite{JQ09}, and spatial information~\cite{ZKZ07} (that is, neighboring pixels are more likely to contain the same materials). 
Although these priors make the corresponding NMF problems more well-posed, the underlying optimization problems to be solved are still computationally difficult (and only local minimum are usually obtained). We refer the reader to the survey~\cite{Jose12} for more details about blind HU. 

In this paper, we make an additional assumption, namely that \emph{most pixels are dominated mostly by one endmember}, and our goal is to cluster the pixels accordingly. In fact, clustering the pixels of a HSI only makes sense for relatively high resolution images. 
For such images, it is often assumed that, for each endmember, there exists at least one pixel containing only that endmember, that is, for all $1 \leq k \leq r$ there exists $j$ such that $M(:,j) = W(:,k)$. 
This is the so-called \emph{pure-pixel assumption}. The pure-pixel assumption is equivalent to the separability assumption (see~\cite{GV12} and the references therein) which makes the corresponding NMF problem tractable, even in the presence of noise~\cite{AGKM11}. Hence, blind HU can be solved efficiently under the pure-pixel assumption. 
Mathematically, a matrix $M \in \mathbb{R}^{m \times n}$ is $r$-separable if it can be written as 
\[
M = WH = W[I_r, H'] \Pi,  
\]
where $W \in \mathbb{R}^{m \times r}$, $H' \geq 0$ and $\Pi$ is a permutation matrix. If $M$ is a HSI, we have, as before, that 
\begin{itemize}
 \item The number $r$ is the number of endmembers present in the HSI. 
 
 \item Each column of $W$ is the spectral signature of an endmember.
 
 \item Each column of $H$ is the abundance vector of a pixel. More precisely, the entry $H(i,j)$ is the abundance of the $i$th endmember in the $j$th pixel. 
 
 \end{itemize} 
Because the column of $H$ sum to one, each column of $M$ belongs to the convex hull of the columns of $W$, that is, 
$\conv(M) \subseteq \conv(W)$. The pure-pixel assumption requires that $\conv(M) = \conv(W)$, that is, that 
 the vertices of the convex hull of the columns of $M$ are the columns of $W$; see the top of Figure~\ref{illus2d} for an illustration in the rank-three case. 
 Hence, the separable NMF problem (or, equivalently, blind HU under the LMM and the pure-pixel assumption) reduces to identifying the vertices of the convex hull of the columns of $M$. However, in noisy settings, this problem becomes more difficult, and although some robust algorithms have been proposed recently (see, e.g.,~\cite{G13} and the references therein), they are typically rather sensitive to noise and outliers. \\

Motivated by the fact that in high-resolution HSI's, most pixels are mostly dominated by one endmember, we develop in this paper a practical and theoretically well-founded hierarchical clustering technique. 
Hierarchical clustering based on NMF has been shown to be faster than flat clustering and can often achieve similar or even better clustering quality \cite{KP13}. At the core of the algorithm is the use of rank-two NMF that splits a cluster into two disjoint clusters. We study the unique property of rank-two NMF as opposed to a higher-rank NMF. We also propose an efficient algorithm for rank-two NMF so that the overall problem of hierarchical clustering of HSI's can be efficiently solved. 

The paper is organized as follows. In Section~\ref{hr2nmf}, we describe our hierarchical clustering approach (see Algorithm~\ref{h2nmfpseudo} referred to as H2NMF). At each step, a cluster is selected (Section~\ref{choice}) and then split into two disjoint clusters (Section~\ref{split}). The splitting procedure has a rank-two NMF algorithm at its core which is described in Section~\ref{r2hsi} where we also provide some sufficient conditions under which the proposed algorithm recovers an optimal solution. In Section~\ref{cg}, we analyze the geometric properties of the hierarchical clustering. 
In Section~\ref{ne}, we show that it outperforms $k$-means, spherical $k$-means (either if they are used in a hierarchical manner, or directly on the full image) and standard NMF on synthetic and real-world HSI's, being more robust to noise, outliers and absence of pure pixels.  
We also show that it can be used as an endmember extraction algorithms and  outperforms vertex component analysis (VCA)~\cite{ND05} and the successive projection algorithm (SPA)~\cite{MC01}, two standard and widely used techniques.

\section{Hierarchical Clustering for HSI's using Rank-Two NMF} \label{hr2nmf}

 As mentioned in the introduction, for high-resolution HSI, one can assume that most pixels contain mostly one material. Hence, given a high-resolution HSI with $r$ endmembers, it makes sense to cluster the pixels into $r$ clusters, each cluster corresponding to one endmember. 
 Mathematically, given the HSI $M \in \mathbb{R}^{m \times n}_+$, 
we want to find $r$ disjoint clusters $\mathcal{K}_k \subset \{1,2,\dots n\}$ for $1 \leq k \leq r$ so that $\cup_{k = 1,2,\dots,r} \mathcal{K}_k = \{1,2,\dots n\}$ and so that all pixels in $\mathcal{K}_k$ are dominated by the same endmember. 

In this paper, we assume the number of endmembers is known in advance. 
In fact, the problem of determining the number of endmembers (also known as model order selection) is nontrivial and out of the scope of this paper; see, e.g., \cite{BN05}. However, a crucial advantage of our approach is that it decomposes the data hierarchically and hence provides the user with a hierarchy of materials (see, e.g., Figures~\ref{clushier} and \ref{clushiersd}). 
In particular, the algorithm does not need to be rerun from scratch if the number of clusters required by the user is modified.   \\

 
In this section, we propose an algorithm to cluster the pixels of a HSI in a hierarchical manner. 
More precisely, at each step, given the current set of clusters $\{ \mathcal{K}_k \}_{k=1}^p$, we select one of the clusters and split it into two disjoint clusters. 
Hierarchical clustering is a standard technique in data mining that organizes a data set into a tree structure of items. It is widely used in text analysis for efficient browsing and retrieval~\cite{cluto, KP13, imagesearch}, as well as exploratory genomic study for grouping genes participating in the same pathway~\cite{genome}. Another example is to segment an image into a hierarchy of regions according to different cues in computer vision such as contours and textures~\cite{bsds500}. In contrast to image segmentation problems, our focus is to obtain a hierarchy of materials from HSI's taken at hundreds of wavelengths instead of the three visible wavelengths.

At each step of a hierarchical clustering technique, one has to address the following two questions: 
\begin{enumerate} 

\item Which cluster should be split next? 

\item How do we split the selected cluster? 

\end{enumerate} 
 These two building blocks for our hierarchical clustering technique for HSI's are described in the following sections.

\subsection{Selecting the Leaf Node to Split} \label{choice}

Eventually, we want to cluster the pixels into $r$ disjoint clusters $\{ \mathcal{K}_k \}_{k=1}^r$, each corresponding to a different endmember. Therefore, each submatrix $M(:,\mathcal{K}_k)$ should be close to a rank-one matrix since for all $j \in \mathcal{K}_k$, we should have $M(:,j) \approx W(:,k)$, possibly up to a scaling factor (e.g., due to different illumination conditions in the image), 
where $W(:,k)$ is the spectral signature of the endmember corresponding to the cluster $\mathcal{K}_k$. 
In particular, in ideal conditions, that is, each pixel contains exactly one material and no noise is present, 
$M(:,\mathcal{K}_k)$ is a rank-one matrix. 
Based on this observation, we define the error $E_k$ corresponding to each cluster as follows 
\[ 
E_k 
\quad = \quad 
\min_{X, \rank(X) = 1} ||M(:,\mathcal{K}_k) - X||_F^2 
\quad = \quad 
 ||M(:,\mathcal{K}_k)||_F^2 - \sigma_1^2(M(:,\mathcal{K}_k)). 
\]
We also define the total error $E = \sum_{k=1}^r E_k$. 
If we decide to split the $k$th cluster $\mathcal{K}_k$ into $\mathcal{K}_k^1$ and $\mathcal{K}_k^2$, the error corresponding to the columns in  $\mathcal{K}_k$ is given by 
\begin{align*}
\left( ||M(:,\mathcal{K}_k^1)||_F^2 - \sigma_1^2(M(:,\mathcal{K}_k^1)) \right) & + \left( ||M(:,\mathcal{K}_k^2)||_F^2 - \sigma_1^2(M(:,\mathcal{K}_k^2))  \right)
  \\
& = \left( ||M(:,\mathcal{K}_k^1)||_F^2 +  ||M(:,\mathcal{K}_k^2)||_F^2 ) \right) - \left( \sigma_1^2(M(:,\mathcal{K}_k^1))  + \sigma_1^2(M(:,\mathcal{K}_k^2))  \right) \\
& = ||M(:,\mathcal{K}_k)||_F^2 - \left( \sigma_1^2(M(:,\mathcal{K}_k^1))  + \sigma_1^2(M(:,\mathcal{K}_k^2))  \right). 
\end{align*}
(Note that the error corresponding to the other clusters is unchanged.) 
Hence, if the $k$th cluster is split, the total error $E$ will be reduced by $\sigma_1^2(M(:,\mathcal{K}_k^1))  + \sigma_1^2(M(:,\mathcal{K}_k^2) - \sigma_1^2(M(:,\mathcal{K}_k))$. Therefore, we propose to split the cluster $k$ for which $\sigma_1^2(M(:,\mathcal{K}_k^1))  + \sigma_1^2(M(:,\mathcal{K}_k^2) - \sigma_1^2(M(:,\mathcal{K}_k))$ is maximized: this leads to the largest possible decrease in the total error $E$ at each step.



\subsection{Splitting a Leaf Node} \label{split}

For the splitting procedure, we propose to use rank-two NMF. Given a nonnegative matrix $M \in \mathbb{R}^{m \times n}_+$, rank-two NMF looks for two nonnegative matrices $W \in \mathbb{R}^{m \times 2}_+$ and $H \in \mathbb{R}^{2 \times n}_+$ such that $WH \approx M$.  
The motivation for this choice is two-fold: 
\begin{itemize}
\item NMF corresponds to the linear mixing model for HSI's (see Introduction), and 
\item Rank-two NMF can be solved efficiently, avoiding the use of an iterative procedure as in standard NMF algorithms. 
In Section~\ref{r2hsi}, we propose a new rank-two NMF algorithm using convex geometry concepts from HSI; see Algorithm~\ref{rank2nmf}. 
\end{itemize} 
Suppose for now we are given a rank-two NMF $(W,H)$ of $M$. Such a factorization is a two-dimensional representation of the data; more precisely, it projects the columns of $M$ onto a two-dimensional pointed cone generated by the columns of $W$. 
 Hence, a naive strategy to cluster the columns of $M$ is to choosing the clusters as follows
 \[
C_1 =  \{ \ i \ | \ H(1,i) \geq H(2,i) \ \} \quad \text{ and } \quad C_2 = \{ \ i \ | \ H(1,i) < H(2,i) \ \}. 
\] 
Defining the vector $x \in [0,1]^n$ as   
 \[
x(i)  = \frac{H(1,i)}{H(1,i)+H(2,i)} \quad \text{ for } \quad 1 \leq i \leq n, 
\] 
the above clustering assignment is equivalent to taking 
 \begin{equation} \label{c1c2}
C_1 =  \{ \ i \ | \ {x}_i \geq \delta \ \} \quad \text{ and } \quad C_2 = \{ \ i \ | \ {x}_i < \delta \ \}, 
 \end{equation}
with $\delta = 0.5$. However, the choice of $\delta = 0.5$ is by no means optimal, and often leads to a rather poor separation. In particular, if an endmember is located exactly between the two extracted endmembers, the corresponding cluster is likely to be divided into two which is not desirable (see Figure~\ref{illus2d}). 
In this section, we present a simple way to tune the threshold $\delta \in [0,1]$ in order to obtain, in general, significantly better clusters $C_1$ and $C_2$.   

Let us define the empirical cumulative distribution of $x$ as follows 
\[
\hat{F}_X(\delta) = \frac{1}{n} \ \big| \{ \ i \ | \ {x}_i \leq \delta \ \} \big|  \in [0,1], \quad \text{ for } \delta \in [0,1].   
\] 
By construction, $\hat{F}_X(0) = 0$ and $\hat{F}_X(1) = 1$.  Let us also define 
\[
\hat{G}_X(\delta) = \frac{1}{n (\bar{\delta}-\underline{\delta})} \ 
\big| 
\{ \ i \ | \  \underline{\delta} =  \max(0,\delta - \hat{\delta}) 
\leq {x}_i \leq 
\min(1,\delta + \hat{\delta}) = \bar{\delta} \ \} \big|  \in [0,1],  
\] 
for $\delta \in [0,1]$, and $\hat{\delta} \in (0,0.5)$ is a small parameter. The function $\hat{G}_X(\delta)$ accounts for the number of points in a small interval around $\delta$. Note that, assuming uniform distribution in the interval $[0,1]$, the  expected value of $\hat{G}_X(\delta)$ is equal to one. In fact, since the entries of $x$ are in the interval $[0,1]$, the expected number of data points in an interval of length $L$ is $n L$. 
In this work, we use $\hat{\delta} = 0.05$.  



Given $\delta$, we obtain two clusters $C_1$ and $C_2$; see Equation~\eqref{c1c2}. We propose to choose a value of $\delta$ such that 
\begin{enumerate} 

\item The clusters are balanced, that is, the two clusters contain, if possible, roughly the same number of elements. Mathematically, we would like that $\hat{F}_X(\delta) \approx 0.5$. 

\item The clustering is stable, that is, if the value of $\delta$ is slightly modified,  then only a few points are transfered from one cluster to the other. 
Mathematically, we would like that $\hat{G}_X(\delta) \approx 0$. 

\end{enumerate} 
We propose to balance these two goals by choosing $\delta$ that minimizes the following criterion: 
\begin{equation} \label{gdelta}
 g(\delta) =  
 \underbrace{- \log \left( \hat{F}_X(\delta) \left(  1-\hat{F}_X(\delta) \right) \right)}_{\text{balanced clusters}} 
 + \underbrace{\exp \left( \hat{G}_X(\delta) \right)}_{\text{stable clusters}}. 
\end{equation}

 The first term avoids skewed classes, while the second promotes a stable clustering. 
Note that the two terms are somewhat well-balanced since, for $\hat{F}_X(\delta) \in [0.1,0.9]$, 
\[
- \log \left( \hat{F}_X(\delta) \left(  1-\hat{F}_X(\delta) \right) \right) \leq 2.5, 
 \] 
 and the expected value of $\hat{G}_X(\delta)$ is one (see above). Note that depending on the application at hand, the two terms of $g(\delta)$ can be balanced in different ways; for example, if one wants to allow very small clusters to be extracted, then the first term of $g(\delta)$ should be given less importance.  

\begin{remark}[Sensitivity to $\delta$] The splitting procedure is clearly very sensitive to the choice of $\delta$. 
For example, as described above, choosing $\delta = 0.5$ can give very poor results. 
However, if the function $g(\delta)$ is chosen in a sensible way, then the corresponding splitting procedure generates in general good clusters.  For example, we had first run all the experiments from Section~\ref{ne} selecting $\delta$ minimizing the function
\[
g(\delta) =  4 \left(  \hat{F}_X(\delta)-0.5 \right)^2 + \left( \hat{G}_X(\delta) \right)^2, 
\]
and it gave very similar results (sometimes slightly better, sometimes slightly worse). 
The advantage of the function~\eqref{gdelta} is that it makes sure no empty cluster is generated 
(since it goes to infinity when $\hat{F}_X(\delta)$ goes to 0 or 1). 
\end{remark}
\begin{remark}[Sensitivity to $\hat{\delta}$] 
The parameter $\hat{\delta}$ is the window size where the stability of a given clustering is evaluated. 
For $\delta$ corresponding to a stable cluster 
(that is, only a few pixels are transferred from one cluster to the other if $\delta$ is slightly modified), 
$\hat{G}_X(\delta)$ will remain small when $\hat{\delta}$ is slightly modified. 
For the considered data sets, most clusterings are stable (because the data is in fact constituted of several clusters of points) hence in that case the splitting procedure does not seem to be very sensitive to $\hat{\delta}$ as long as it is in a reasonable range. In fact, we have also run the numerical experiments for $\hat{\delta} =0.01$ and $\hat{\delta} =0.1$ and it gave very similar results 
(in particular, for the Urban, San Diego, Terrain and Cuprite HSI's from Section~\ref{realworld}, it is hardly possible to distinguish the solutions with the naked eye). 
\end{remark}

Figure~\ref{illus2d} illustrates the geometric insight behind the splitting procedure in the case $r=3$ (see also Section~\ref{cg}),  while Algorithm~\ref{h2nmfpseudo} gives a pseudo-code of the full hierarchical procedure. 
\begin{figure}[ht!] 
\centering 
   \includegraphics[width=8.5cm]{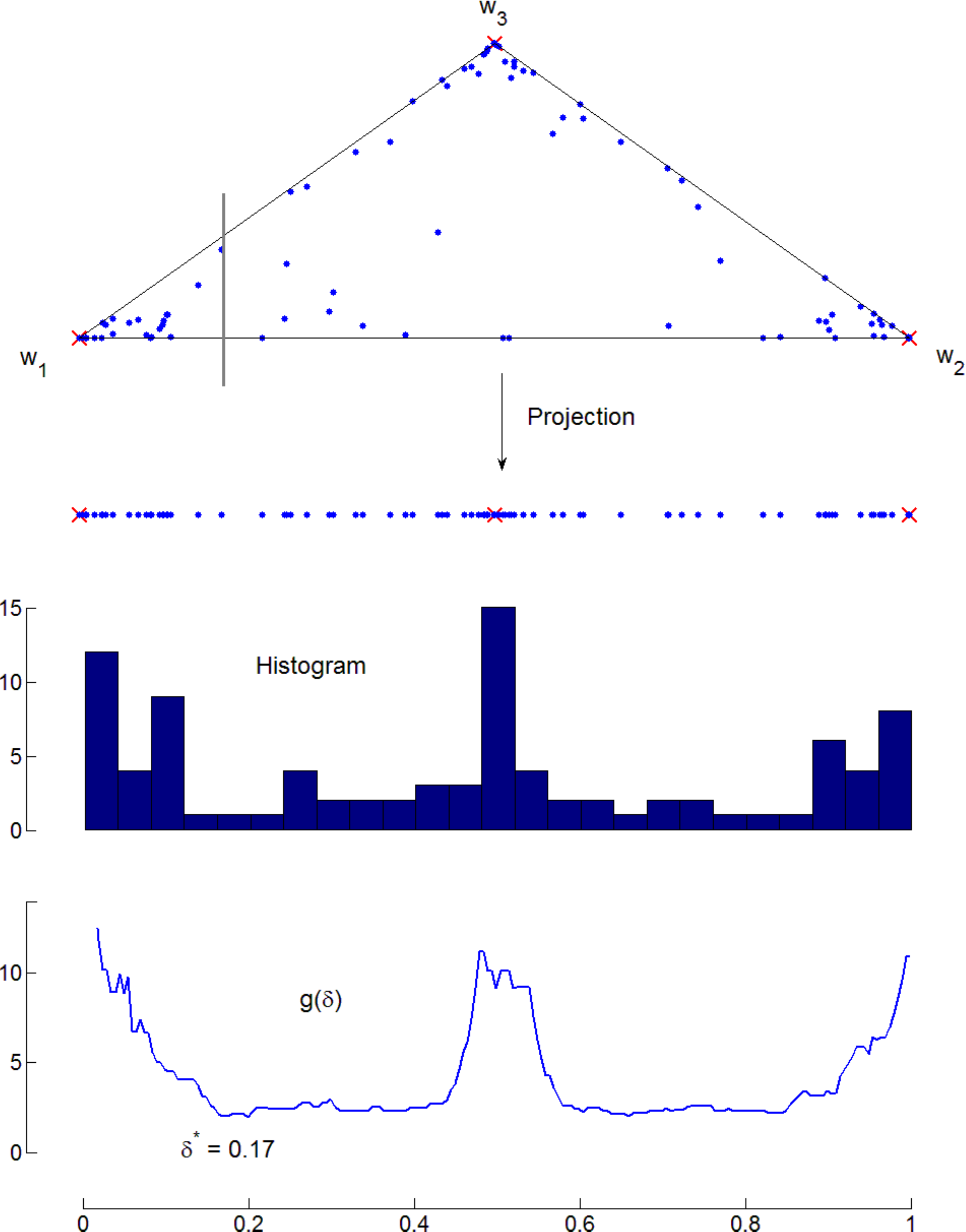} 
\caption{Illustration of the splitting technique based on rank-two NMF.}
\label{illus2d}
\end{figure}

\algsetup{indent=2em}
\begin{algorithm}[ht!]
\caption{Hierachical Clustering of a HSI based on Rank-Two NMF (H2NMF) \label{h2nmfpseudo}}
\begin{algorithmic}[1] 
\REQUIRE A HSI $M \in \mathbb{R}^{m \times n}_+$ and the number $r$ of clusters to generate. 
\ENSURE Set of disjoint clusters $\mathcal{K}_i$ for $1 \leq i \leq r$ with $\cup_i \mathcal{K}_i = \{1,2,\dots,n\}$. 
    \medskip 
		
\STATE \emph{\% Initialization} 
\STATE $\mathcal{K}_1 = \{1,2,\dots,n\}$ and $\mathcal{K}_i = \emptyset$ for $2 \leq i \leq r$. 
\STATE $\left( \mathcal{K}_{1}^1, \mathcal{K}_{1}^2 \right)$ = splitting($M, \mathcal{K}_{1}$). \emph{\% See Algorithm~\ref{splitalgo} and Section~\ref{split}}
\STATE $\mathcal{K}_i^1 = \mathcal{K}_i^2 = \emptyset$ for $2 \leq i \leq r$. 
\FOR {$k = 2$ : $r$}
		
		\STATE \emph{\% Select the cluster to split; see Section~\ref{choice}}
		\STATE Let $j = \argmax_{i=1,2,\dots r} 
						\sigma_1^2(M(:,\mathcal{K}_i^1))  + \sigma_1^2(M(:,\mathcal{K}_i^2) - \sigma_1^2(M(:,\mathcal{K}_i))$. 

		\STATE \emph{\% Update the clustering}
		\STATE $\mathcal{K}_j = \mathcal{K}_j^1$ and $\mathcal{K}_{k} = \mathcal{K}_j^2$. 
		
			\STATE \emph{\% Split the new clusters (Algorithm~\ref{splitalgo})}  
			\STATE $\left( \mathcal{K}_j^1, \mathcal{K}_j^2 \right)$ = splitting($M, \mathcal{K}_j$) 
			and 
			$\left( \mathcal{K}_k^1, \mathcal{K}_k^2 \right)$ = splitting($M, \mathcal{K}_k$).  
			
\ENDFOR

\end{algorithmic}
\end{algorithm}

\algsetup{indent=2em}
\begin{algorithm}[ht!]
\caption{Splitting of a HSI using Rank-Two NMF \label{splitalgo}}
\begin{algorithmic}[1] 
\REQUIRE A HSI $M \in \mathbb{R}^{m \times n}_+$ and a subset $\mathcal{K} \subseteq \{1,2,\dots,n\}$. 
\ENSURE Set of two disjoint clusters $\mathcal{K}^1$ and $\mathcal{K}^2$  with $\mathcal{K}_1 \cup \mathcal{K}_2= \mathcal{K}$. 
    \medskip 
		
\STATE Let $(W,H)$ be the rank-two NMF of $M(:,\mathcal{K})$ computed by Algorithm~\ref{rank2nmf}.  

\STATE Let $x(i)  = \frac{H(1,i)}{H(1,i)+H(2,i)}$  for  $1 \leq i \leq |\mathcal{K}|$. 

\STATE Compute $\delta^*$ as the minimum of $g(\delta)$ defined in \eqref{gdelta}. 

\STATE $\mathcal{K}^1 = \{ \  \mathcal{K}(i) \ | \ x(i) \geq \delta^* \ \}$ 
and 
$\mathcal{K}^2 = \{ \ \mathcal{K}(i) \ | \ x(i) < \delta^* \ \}$.  

\end{algorithmic}
\end{algorithm}

\subsection{Rank-Two NMF for HSI's} \label{r2hsi}

In this section, we propose a simple and fast algorithm for the rank-two NMF problem tailored for HSI's (Section~\ref{desalgo}). 
Then, we discuss some sufficient conditions for the algorithm to be optimal (Section~\ref{tg}).

\subsubsection{Description of the Algorithm} \label{desalgo}

When a nonnegative matrix $M \in \mathbb{R}^{m \times n}_+$ has rank two, Thomas has shown~\cite{Tho} that finding two nonnegative matrices $(W,H) \in \mathbb{R}^{m \times 2}_+ \times \mathbb{R}^{2 \times n}_+$ such that $M = WH$ is always possible (see also~\cite{CR93}). 
    This can be 
    explained geometrically as follows: viewing columns of $M$ as points in $\mathbb{R}^m_+$, the fact that $M$ has rank two implies that the set of its columns belongs to a two-dimensional subspace. Furthermore, because these columns are nonnegative, they belong to a two-dimensional pointed cone, see Figure~\ref{rank2}. Since such a cone is always spanned by two extreme vectors, this implies that all columns of $M$ can be represented exactly as nonnegative linear combinations of two nonnegative vectors, and therefore the exact NMF is always possible\footnote{The reason why this property no longer holds for higher values of the rank $r$ of matrix $M$ is that a $r$-dimensional cone is not necessarily spanned by a set of $r$ vectors when $r > 2$.} for $r=2$.
\begin{figure}[ht!]
\begin{center}
\includegraphics[width=6cm]{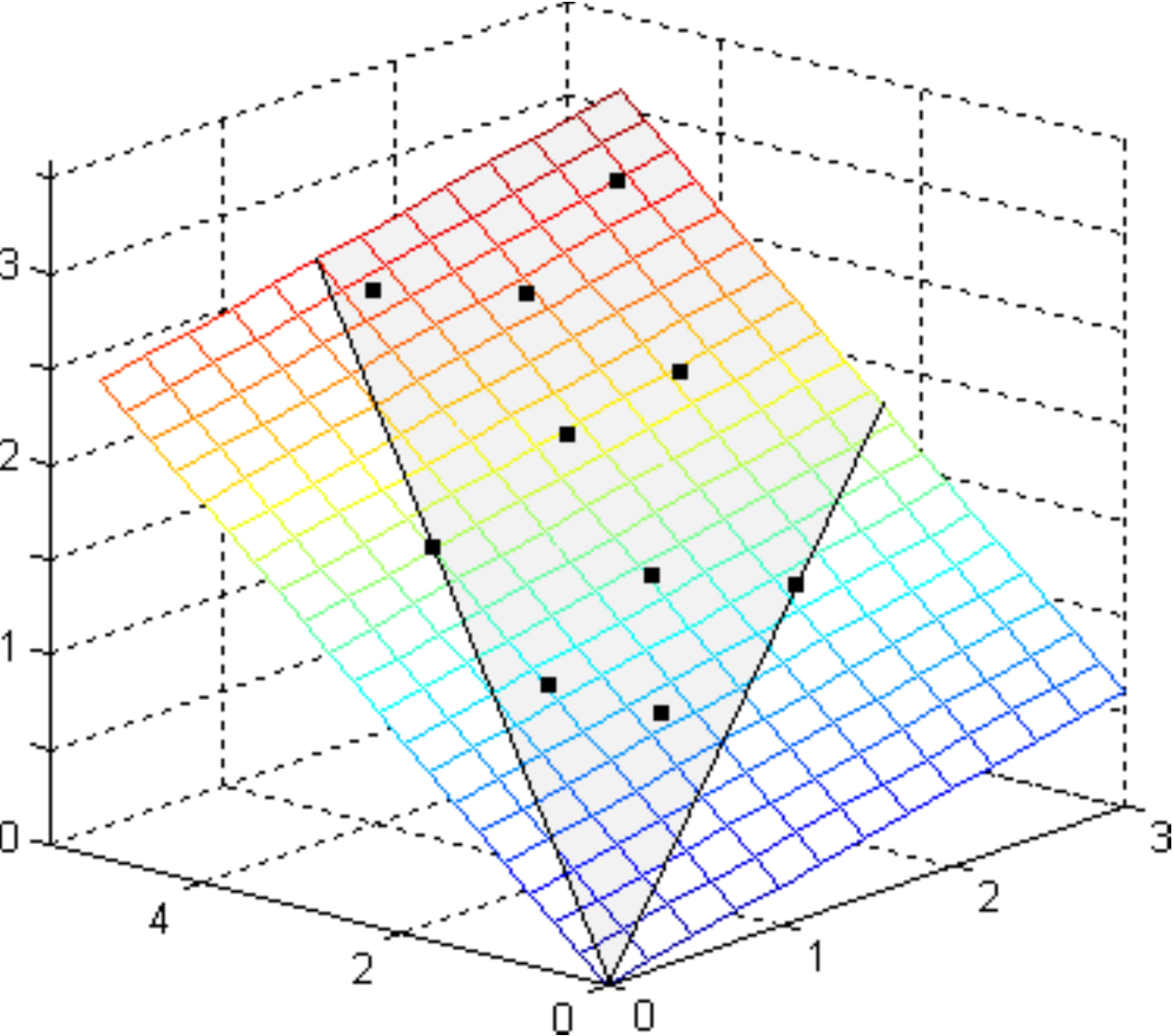}
\caption{Illustration of exact NMF for a rank-two $3$-by-$10$ nonnegative matrix \cite[p.24]{G11}.}
\label{rank2}
\end{center}
\end{figure}
    Moreover, these two extreme columns can easily be identified.  For example, if the columns of $M$ are normalized so that their entries sum to one, then the columns of $M$ belong to a line segment and it is easy to detect the two vertices. This can be done for example using any endmember extraction algorithm under the linear mixing model and the pure-pixel assumption since they aim to 
		detect the vertices (corresponding to the endmembers) of a convex hull of a set of points (see Introduction). 
		In this paper, we use the successive projection algorithm  (SPA)~\cite{MC01} which is a highly efficient and widely used algorithm; see Algorithm~\ref{spa}. 
		\algsetup{indent=2em}
\begin{algorithm}[ht!]
\caption{Successive Projection Algorithm (SPA) 
\label{spa}~\cite{MC01, GV12}}
\begin{algorithmic}[1] 
\REQUIRE Separable matrix $M = W[I_r, H'] \Pi$ where $H' \geq 0$, the sum of the entries of each column of $H'$ is smaller than one, $W$ is full rank and $\Pi$ is a permutation, and the number $r$ of columns to be extracted. 
\ENSURE Set of indices $K$ such that $M(:,K) = W$ (up to permutation). 
    \medskip 
\STATE Let $R = M$, 
$K = \{\}$. 
\FOR {$i = 1$ $:$ $r$}   
\STATE $k = \argmax_j ||R_{:j}||_2$.  
\STATE $R \leftarrow \left(I-\frac{R_{:k} R_{:k}^T}{||R_{:k}||_2^2}\right)R$. \vspace{0.1cm} 
\STATE $K = K \cup \{k\}$. 
\ENDFOR
\end{algorithmic}
\end{algorithm} 
 Moreover, it has been shown to be robust to any small perturbation of the input matrix~\cite{GV12}. Note that SPA is closely related to the automatic target generation process algorithm (ATGP)~\cite{RC03} and the successive volume maximization algorithm (SVMAX)~\cite{CM11}; see~\cite{Ma13} for a survey about these methods. 
Note that it would be possible to use more sophisticated endmember extraction algorithms for this step, e.g., RAVMAX~\cite{AC10} or WAVMAX~\cite{CM11} which are more robust variants of SPA (although computationally much more expensive). 

We can now describe our proposed rank-two NMF algorithm for HSI: It first projects the columns of $M$ into a two-dimensional linear space using the SVD (note that if the rank of input matrix is two, this projection step is exact), then identifies two important columns with SPA and projects them onto the nonnegative orthant, and finally computes the optimal weights solving a nonnegative least squares problem (NNLS); see Algorithm~\ref{rank2nmf}. 
\algsetup{indent=2em}
\begin{algorithm}[!h]
\caption{Rank-Two NMF for HSI's \label{rank2nmf}}
\begin{algorithmic}[1] 
\REQUIRE A nonnegative matrix $M \in \mathbb{R}^{m \times n}_+$. 
\ENSURE A rank-two NMF $(W,H) \in \mathbb{R}^{m \times 2}_+ \times \mathbb{R}^{2 \times n}_+$. 
\STATE \% \emph{Compute an optimal rank-two approximation of $M$ }
\STATE $[U,S,V^T] =$ svds$(M,2)$; \% \emph{See the Matlab function \texttt{svds} }
\STATE Let $X = S V$ $(= U^T U S V = U^TM)$;
\STATE \% \emph{Extract two indices using SPA}
\STATE $K =$ SPA$(X,2)$; \% \emph{See Algorithm~\ref{spa}}
\STATE $W = \max\big(0,USV(:,K) \big)$; 
\STATE $H = \argmin_{Y \geq 0} ||M-WY||_F^2$; \% \emph{See Algorithm~\ref{nnls2}  }
\end{algorithmic} 
\end{algorithm} 
\algsetup{indent=2em}
\begin{algorithm}[ht!]
\caption{Nonnegative Least Squares with Two Variables \label{nnls2}~\cite{KP13}}
\begin{algorithmic}[1] 
\REQUIRE A matrix $A \in \mathbb{R}^{m \times 2}$ and a vector $b \in \mathbb{R}^{m}$. 
\ENSURE A solution $x \in \mathbb{R}^{2}_+$ to $\min_{x \geq 0} ||Ax - b||_2$. 
    \medskip 
\STATE \% \emph{Compute the solution of the unconstrained least squares problem }
\STATE $x = \argmin_x ||Ax - b||_2$ (e.g., solve the normal equations $(A^T A) x = A^T b$). 
\IF{$x \geq 0$}   
\STATE return. 
\ELSE 
\STATE \% \emph{Compute the solutions for $x(1) = 0$ and $x(2) = 0$ (the two possible active sets)} \vspace{0.1cm}
\STATE Let $y = \left(0, \max\left(0, \frac{A(:,1)^T b}{||A(:,1)||_2^2}\right) \right)$ 
and $z = \left(  \max\left(0, \frac{A(:,2)^T b}{||A(:,2)||_2^2}\right), 0\right)$. \vspace{0.1cm} 
 \IF {$||Ay - b||_2 < ||Az - b||_2$} 
 \STATE $x = y$.
 \ELSE 
\STATE $x = z$. 
\ENDIF 
\ENDIF 
\end{algorithmic}
\end{algorithm}

Let us analyze the computational cost of Algorithm~\ref{rank2nmf}. 
The computation of the rank-two SVD of $M$ is $\mathcal{O}(mn)$ operations~\cite{GV96}. 
(Note that this operation scales well for sparse matrices as there exist SVD methods that can handle large sparse matrices, e.g., the \texttt{svds} function of Matlab.) 
For HSI's, $m$ is much smaller than $n$ (usually $m \sim 200$ while $n \sim 10^6$) hence it is faster to computing the SVD of $M$ using the SVD of $MM^T$ which requires $2mn + O(m^2)$ operations; see, e.g.,~\cite{ND05}. 
Note however that this is numerically less stable as the condition number of the corresponding problem is squared. 
Extracting the two indices in step 5 with SPA requires $\mathcal{O}(n)$ operations~\cite{GV12}, while computing the optimal $H$ requires solving $n$ linear systems in two variables for a total computational cost of $\mathcal{O}(mn)$ operations~\cite{KP13}. 
In fact, the NNLS  
\[
\min_{X \in \mathbb{R}^{2 \times n}_+} ||M-WX||_F^2
\]
where $W \in \mathbb{R}^{m \times 2}_+$ can be decoupled into $n$ independent NNLS in two variables since 
\[
||M-WX||_F^2 = \sum_{i=1}^n ||M(:,i) - WX(:,i)||_2^2. 
\] 
Algorithm~\ref{nnls2} implements the algorithm in~\cite{KP13} to solve these subproblems. 
 
Finally, Algorithm~\ref{rank2nmf} requires $\mathcal{O}(mn)$ operations which implies 
that the global hierarchical clustering procedure (Algorithm~\ref{h2nmfpseudo}) requires at most $\mathcal{O}(mnr)$ operations.  
Note that this is rather efficient and developing a significantly faster method would be difficult. 
In fact, it already requires $\mathcal{O}(mnr)$ operations to compute the product of $W \in \mathbb{R}^{m \times r}$ and $H \in \mathbb{R}^{r \times n}$, or to assign optimally $n$ data points in dimension $m$ to $r$ cluster centroids using the Euclidean distance. 
Note however that in an ideal case, if the largest cluster is always divided into two clusters containing the same number of pixels (hence we would have a perfectly balanced tree), the number of operations reduces to $\mathcal{O}(mn \log(r))$. Hence, in practice, if the clusters are well balanced, the computational cost is rather in $\mathcal{O}(mn \log(r))$ operations. 


\subsubsection{Theoretical Motivations} \label{tg} 

An mentioned above, rank-two NMF can be solved exactly for rank-two input matrices. 
Let us show that Algorithm~\ref{rank2nmf} does. 
\begin{theorem} \label{optr2}
If $M$ is a rank-two nonnegative matrix whose entries of each column sum to one, 
then Algorithm~\ref{rank2nmf} computes an optimal rank-two NMF of $M$. 
\end{theorem}
\begin{proof}
Since $M$ has rank-two and is nonnegative, there exists an exact rank-two NMF $(F,G)$ of $M = FG = F(:,1)G(1,:) + X(:,2)Y(2,:)$~\cite{Tho}. 
Moreover, since the entries of each column of $M$ sum to one, we can assume without loss of generality that the entries of the each column of $F$ and $G$ sum to one as well. In fact, we can normalize the two columns of $F$ so that their entries sum to one while scaling the rows of $G$ accordingly: 
\[
M = 
\underbrace{ \frac{F(:,1)}{||F(:,1)||_1} }_{F'(:,1)} 
\underbrace{||F(:,1)||_1 G(1,:)}_{G'(1,:)}
+ \underbrace{\frac{ F(:,2) }{ ||F(:,2)||_1 } }_{F'(:,2)}
\underbrace{||F(:,2)||_1 G(2,:)}_{G'(2,:)} = F' G'. 
\] 
Since the entries of each column of $M$ and $F'$ sum to one and $M=F'G'$, the entries of each column of $G'$ have to sum to one as well. Hence, the columns of $M$ belong to the line segment $[F'(:,1),F'(:,2)]$. 

Let $(U,S,V^T)$ be the rank-two SVD of $M$ computed at step~2 of Algorithm~\ref{rank2nmf}, we have 
$SV = U^T M = (U^T F') G'$. Hence, the columns of $SV$ belong to the line segment \mbox{$[U^TF'(:,1)$},\mbox{$U^TF'(:,2)]$} so that SPA applied on $SV$ will identify two indices corresponding to two columns of $M$ being the vertices of the line segment defined by its columns~\cite[Th.1]{GV12}. 
Therefore, any column of $M$ can be reconstructed with a convex combination of these two extracted columns and Algorithm~\ref{rank2nmf} will generate an exact rank-two NMF of~$M$. 
\end{proof}

		\begin{corollary}
		Let $M$ be a noiseless HSI with two endmembers satisfying the linear mixing model and the sum-to-one constraint, 
		then Algorithm~\ref{rank2nmf} computes an optimal rank-two NMF of $M$. 
		\end{corollary}
		\begin{proof}
By definition, $M = WH$ where the columns of $W$ are equal to the spectral signatures of the two endmembers and the columns of $H$ are nonnegative and sum to one (see Introduction). The rest of the proof follows from the second part of the proof of Theorem~\ref{optr2}.  (Note that the pure-pixel assumption is not necessary.) 
\end{proof}

In practice, the sum-to-one constraints assumption is sometimes relaxed to the following: the sum of the entries of each column of $H$ is at most one. This has several advantages such as allowing the image to containing `background' pixels with zero spectral signatures, or taking into account different intensities of light among the pixels in the image; see, e.g.,~\cite{Jose12}. In that case, Algorithm~\ref{rank2nmf} works under the additional pure-pixel assumption: 
\begin{corollary}
		Let $M$ be a noiseless HSI with different illumination conditions, with two endmembers, and satisfying the linear mixing model and the 
		pure-pixel assumption, then Algorithm~\ref{rank2nmf} computes an optimal rank-two NMF of $M$. 
		\end{corollary}
		\begin{proof}
		By assumption,  $M = W[I_2, H']\Pi$ where $H'$ is nonnegative and the entries of each column sum to at most one, 
		and $\Pi$ is a permutation. 
		This implies that the columns of $M$ are now in the triangle whose vertices are $W(:,1)$, $W(:,2)$ and the origin. Following the proof of Theorem~\ref{optr2}, after the SVD, the columns of $SV$ are in the triangle whose vertices are $U^TW(:,1)$, $U^TW(:,2)$ and the origin. 
		Hence SPA will identify correctly the indices corresponding to $W(:,1)$ and $W(:,2)$~\cite[Th.1]{GV12} so that any column of $M$ can be reconstructed using these two columns. 
\end{proof}

\vspace{0.2cm}

At the first steps of the hierarchical procedure, rank-two NMF maps the data points into a two-dimensional subspace. However, the input matrix does not have rank-two if it contains more than two endmembers.  
In the following, we derive some simple sufficient conditions to support the fact that the rank-two SVD of a nonnegative matrix is nonnegative (or at least has most of its entries nonnegative). 
Let us refer to an optimal rank-two approximation of a matrix $M$ as an optimal solution of 
 \begin{equation} \nonumber 
 \min_{A \in \mathbb{R}^{m \times n}} ||M-A||_F^2 \quad \text{ such that }  \quad \rank(A) = 2. 
 \end{equation}
 We will also refer to rank-two NMF as the following optimization problem 
 \begin{equation} \nonumber 
 \min_{U \in \mathbb{R}^{m \times 2}, V \in \mathbb{R}^{2 \times n}} ||M-UV||_F^2 
\quad 
\text{ such that }  
\quad 
U \geq 0 \text{ and } V \geq 0.  
 \end{equation}
  
\begin{lemma} \label{lem2} Let $M \in \mathbb{R}^{m \times n}_+$, $A \in \mathbb{R}^{m \times n}$ be an optimal rank-two approximation of $M$, and $R=M-A$ be the residual error. If 
\[
L = \min_{i,j}(M_{ij}) \geq \max_{i,j} R_{ij}, 
\]
then every entry of $A$ is nonnegative. 
\end{lemma}
\begin{proof}
If $A_{kl} < 0$ for some $(k,l)$, then $L \leq M_{kl} < M_{kl} - A_{kl} = R_{kl} \leq \max_{ij} R_{ij}$, a contradiction. 
\end{proof}

\begin{corollary} \label{cor2} Let $M \in \mathbb{R}^{m \times n}_+$ satisfy 
\[
L = \min_{i,j}(M_{ij}) \geq \sigma_3(M). 
\]
Then any optimal rank-two approximation of $M$ is nonnegative. 
\end{corollary}
\begin{proof}
This follows from Lemma  \ref{lem2} since, for any optimal rank-two approximation $A$ of $M$ with $R = M-A$, we have 
$\max_{ij} R_{ij} \leq ||R||_2 = \sigma_3(M)$. 
\end{proof}
Corollary~\ref{cor2} shows that a positive matrix close to having rank two and/or only containing relatively large entries is likely to have an optimal rank-two approximation which is nonnegative. Note that HSI's usually have mostly positive entries and, in fact, we have observed that the best rank-two approximation of real-world HSI's typically contains mostly nonnegative entries 
(e.g., 
for the Urban HSI more than 99.5\%, 
for the San Diego HSI more than 99.9\%,  
for the Cuprite HSI more than 99.98\%, and for the Terrain HSI more than 99.8\%; 
see Section~\ref{realworld} for a description of these data sets). 
It would be interesting to investigate further sufficient and necessary conditions for the optimal rank-two approximations of a nonnegative matrix to be nonnegative; this is a topic for further research. 
Note also Theorem~\ref{optr2} only holds for rank-two NMF and cannot be extended to more general cases with an arbitrary $r$. 
Consequently, we designed Algorithm~\ref{rank2nmf} specifically for rank-two NMF. However, Algorithm~\ref{rank2nmf} is important in the context of hierarchical clustering where rank-two NMF is the core computation. 
We will show in Section~\ref{ne} that our overall method achieves high efficiency compared to other hyperspectral unmixing methods. Moreover, if we flatten the obtained tree structure and look at the clusters corresponding to the leaf nodes, we will see that H2NMF achieves much better cluster quality compared to the flat clustering methods including $k$-means and spherical $k$-means. Thus, though the theory in this paper is developed for rank-two NMF only, it has great significance in clustering hyperspectral images with more than two endmembers.


\subsection{Geometric Interpretation of the Splitting Procedure} \label{cg}

Given a HSI $M \in \mathbb{R}^{m \times n}$ containing $r$ endmembers, and given that the pure-pixel assumption holds, we have that 
\[
M = WH = W[I_r, H'] \Pi,  
\]
where $W \in \mathbb{R}^{m \times r}$, $H' \geq 0$ and $\Pi$ is a permutation matrix. This implies that the convex hull $\conv(M)$ of the columns of $M$ coincides with the convex hull of the columns of $W$ and has $r$ vertices; see Introduction. 
A well-known fact in convex geometry is that the projection of any polytope $P$ into an affine subspace generates another polytope, say $P'$. Moreover, each vertex of $P'$ results from the projection of at least one vertex of $P$ (although it is unlikely, it may happen that two vertices are projected onto the same vertex, given that the projection is parallel to the segment joining these two vertices). 
It is interesting to notice that this fact has been used previously in hyperspectral imaging: for example, the widely used VCA algorithm~\cite{ND05} uses three kinds of projections: First, it projects the data into a $r$-dimensional space using the SVD (in order to reduce the noise). 
Then, at each step,  
\begin{itemize}
\item In order to identify a vertex (that is, an endmember), VCA projects $\conv(M)$ onto a one-dimensional subspace. More precisely, it randomly generates a vector $c \in \mathbb{R}^{m}$ and then selects the columns of $M$ maximizing $c^TM(:,i)$. 

\item It projects all columns of $M$ onto the orthogonal complement of the extracted vertex so that, if $W$ is full rank (that is, if $\conv(M)$ has $r$ vertices and has dimension $r-1$), the projection of $\conv(M)$ has $r-1$ vertices and has dimension $r-2$ (this step is the same as step 4 of SPA; see Algorithm~\ref{spa}). 
\end{itemize}


In view of these observations, Algorithm~\ref{rank2nmf} can be geometrically interpreted as follows:  
\begin{itemize}
\item At the first step, the data points are projected into a two-dimensional subspace so that the maximum variance is preserved. 

\item At the second step, two vertices are extracted by SPA. 

\item At the third step, the data points are projected onto the two-dimensional convex cone generated by these two vertices. 

\end{itemize}

\subsection{Related Work} \label{rz} 

It has to be noted that the use of rank-two NMF as a subroutine to solve classification problems has already been studied before.  
In~\cite{LSVH13}, a hierarchical NMF algorithm was proposed (namely, hierarchical NMF) based on rank-two NMF, and was used to identify tumor tissues 
in magnetic resonance spectroscopy images of the brain. 
The rank-two NMF subproblems were solved via standard iterative NMF techniques. In~\cite{hchnmf}, a hierarchical approach was proposed for convex-hull NMF, that could discover clusters not corresponding to any vertex of the $\text{conv}(M)$ but lying inside $\text{conv}(M)$, and an algorithm based on FastMap~\cite{fastmap} was used. 
In~\cite{KP13}, hierarchical clustering based on rank-two NMF was used for document classification. The rank-two subproblems were solved using alternating nonnegative least squares~\cite{KP07, KP11}, that is, by optimizing alternatively $W$ for $H$ fixed, and $H$ for $W$ fixed (the subproblems being efficiently solved using Algorithm~\ref{nnls2}). \\

However, these methods do not take advantage of the nice properties of rank-two NMF, and the novelty of our technique is threefold:  
\begin{itemize}

\item The way the next cluster to be split is chosen based on a greedy approach (so that the largest possible decrease in the error is obtained  at each step); see Section~\ref{choice}.

\item The way the clusters are split based on a trade off between having balanced clusters and stable clusters; see Section~\ref{split}.

\item The use of a rank-two NMF technique tailored for HSI's (using their convex geometry properties) to design a splitting procedure; see Section~\ref{r2hsi}.

\end{itemize}

\section{Numerical Experiments} \label{ne}

In the first part, we compare different algorithms on synthetic data sets: this allows us to highlight their differences and also shows that our hierarchical clustering approach based on rank-two NMF is rather robust to noise and outliers. 
In the second part, we apply our technique to real-world hyperspectral data sets. This in turn shows the power of our rank-two NMF approach for clustering, but also as a robust hyperspectral unmixing algorithm for HSI.  
The Matlab code is available at \url{https://sites.google.com/site/nicolasgillis/}.  
All tests are preformed using Matlab on a laptop Intel CORE i5-3210M CPU @2.5GHz 2.5GHz 6Go RAM.

\subsection{Tested Algorithms} 

We will compare the following algorithms: 
\begin{enumerate}
\item \textbf{H2NMF}: hierarchical clustering based on rank-two NMF; see Algorithm~\ref{h2nmfpseudo} and Section~\ref{hr2nmf}. 

\item \textbf{HKM}: hierarchical clustering based on $k$-means. This is exactly the same algorithm as H2NMF except that the clusters are split using $k$-means instead of the rank-two NMF based technique described in Section~\ref{split} (we used the \texttt{kmeans} function of Matlab). 

\item \textbf{HSPKM}: hierarchical clustering based on spherical $k$-means~\cite{BDGS03}. This is exactly the same algorithm as H2NMF except that the clusters are split using spherical $k$-means (we used a Matlab code available online\footnote{\url{http://www.mathworks.com/matlabcentral/fileexchange/28902-spherical-k-means/content/spkmeans.m}}). 

\item \textbf{NMF}: we compute a rank-$r$ NMF $(U,V)$ of the HSI $M$ using the accelerated HALS algorithm from~\cite{GG11}. Each pixel is assigned to the cluster corresponding to the largest entry of the columns of~$V$.  

\item \textbf{KM}: $k$-means algorithm with $k=r$.

\item \textbf{SPKM}: spherical $k$-means algorithm with $k=r$. 
\end{enumerate} 

Moreover, the cluster centroids of HKM and HSPKM are initialized the same way as for H2NMF, 
that is, using steps 2-5 of Algorithm~\ref{rank2nmf}. 
NMF, KM and SPKM are initialized in a similar way: the rank-$r$ SVD of $M$ is first computed (which reduces the noise) and then SPA is applied on the resulting low-rank approximation of $M$ (this is essentially equivalent to steps 2-5 of Algorithm~\ref{rank2nmf} but replacing 2 by $r$). 
Note that we have tried using random initializations for HKM, HSPKM, NMF, KM and SPKM (which is the default in Matlab) but the corresponding clustering results were very poor (for example, NMF, KM and SPKM were in general not able to identifying the clusters perfectly in noiseless conditions).  Recall that SPA is optimal for HSI's satisfying the pure-pixel assumption~\cite{GV12} hence it is a reasonable initialization.

\subsection{Synthetic Data Sets}  \label{sd}

In this section, we compare the six algorithms described in the previous section on synthetic data sets, so that the ground truth labels are known. 
Given the parameters $\epsilon \geq 0$, $s \in \{0,1\}$ and $b \in \{0,1\}$, the synthetic HSI $M = [WH, \; Z] + N$ with $W \in \mathbb{R}^{m \times r}_+$, $H \in \mathbb{R}^{r \times (n-z)}_+$, $Z \in \mathbb{R}^{m \times z}_+$ and $N \in \mathbb{R}^{m \times n}$ is generated as follows: 
\begin{itemize}

\item We use six endmembers, that is, $r=6$. 

\item The spectral signatures of the six endmembers, that is, the columns of $W$, are taken as the spectral signatures of materials from the Cuprite HSI (see Section~\ref{cuprite}) and we have $W~\in~\mathbb{R}^{188 \times 6}_+$; see Figure~\ref{cup6}. 
Note that $W$ is rather poorly conditioned ($\kappa(W) = 91.5$) as the spectral signatures look very similar to one another.
\begin{figure}[ht] 
\centering 
\includegraphics[width=6cm]{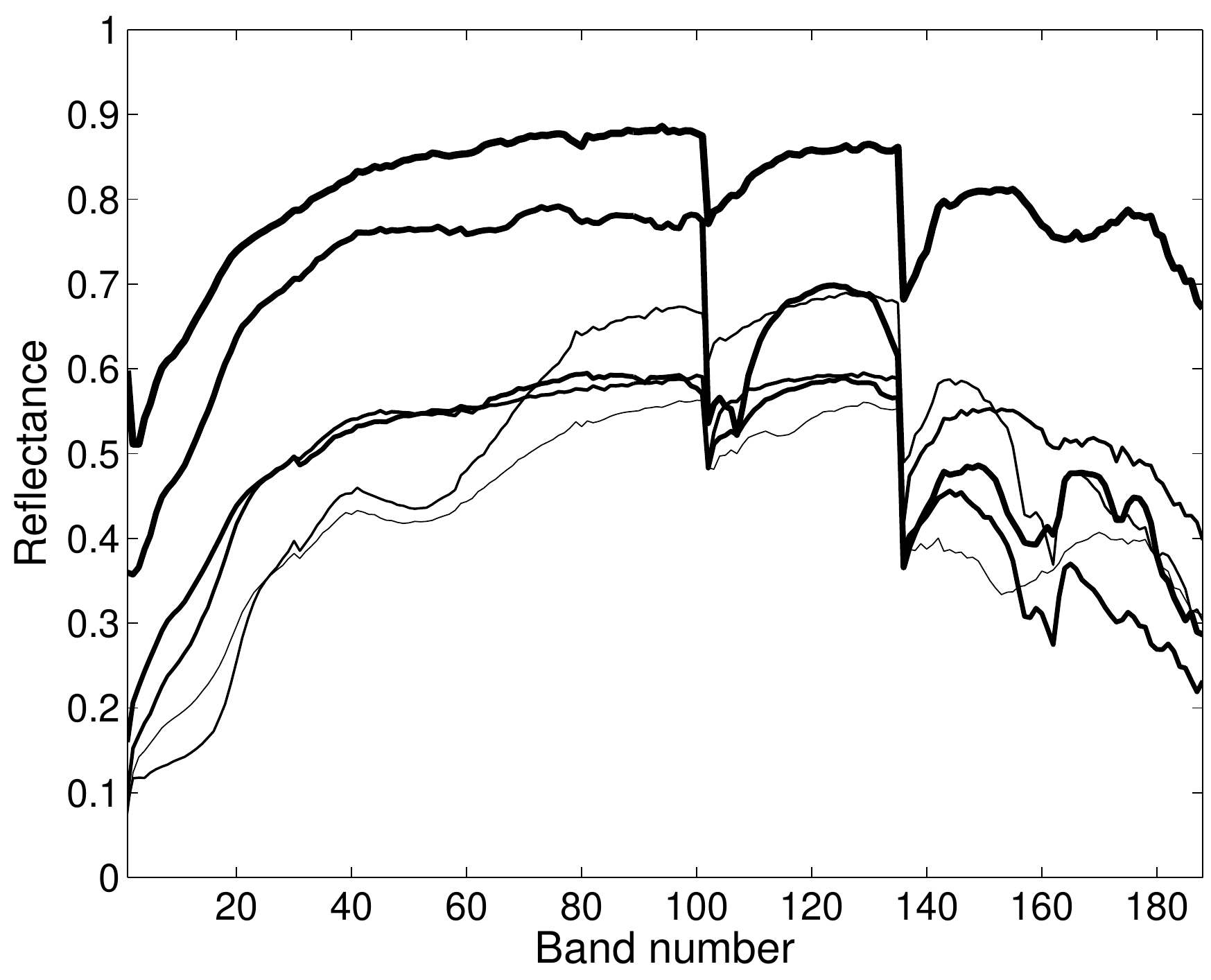}
\caption{Spectral signatures from the Cuprite HSI used for the synthetic data sets.}
\label{cup6}
\end{figure}

\item The pixels are assigned to the six clusters  $\mathcal{K}_k$ $1 \leq k \leq r$ where each cluster contains a different number of pixels with
$|\mathcal{K}_k| = 500-(k-1)50$, $1 \leq k \leq r$ (for a total of 2250 pixels). 


\item Once a pixel, say the $i$th, has been assigned to a cluster, say the $k$th, the corresponding column of $H$ is generated as follows: 
\[
H(:,i) = 0.9 \, e_k + 0.1 \, x , 
\]
where $e_k$ is the $k$th column of the identity matrix, and $x \in \mathbb{R}^{r}_+$ is drawn from a Dirichlet distribution where all parameters are equal to $0.1$. Note that the Dirichlet distribution generates a vector $x$ whose entries sum to one (hence the entries of $H(:,i)$ also do), while the weight of the entries of $x$ is concentrated only in a few components (hence each pixel usually contains only a few endmembers in large proportions). 
In particular, each pixel contains at least 90\% of a single endmember.

\item If $s=1$, each column of $H$ is multiplied by a constant drawn uniformly at random between 0.8 and 1. This allows us to take into account different illumination conditions in the HSI. Otherwise, if $s=0$, then $H$ is not modified.

\item If $b=1$, then ten outliers and forty background pixels with zero spectral signatures are added to $M$, that is, $Z = [z_1, \, z_2, \, \dots, \, z_{10}, \, 0_{m \times 40} ]$ where $0_{p \times q}$ is the $p$-by-$q$ all zero matrix. Each entry of an outlier $z_p \in \mathbb{R}^{m}_+$ ($1 \leq p \leq 10$) is drawn uniformly at random in the interval $[0,1]$ (using the \texttt{rand} function of Matlab), and then the $z_p$'s are scaled as follows:  
 \[
 z_p \leftarrow  K_W \, \frac{z_p}{||z_p||_2} \quad 1 \leq  p \leq 10, 
 \]
 where $K_W = \frac{1}{r} \sum_{k=1}^r ||W(:,k)||_2$ is the average of the norm of the columns of $W$. 
If $b=0$, no outliers nor background pixels with zero spectral signatures are added to $M$, that is, $Z$ is the empty matrix. 

\item The $j$th column of the noise matrix $N$ is generated as follows: each entry is generated following the normal distribution
    $N(i,j) \sim \mathcal{N}(0,1)$ for all $i$ (using the \texttt{randn} function of Matlab) and is then scaled as follows 
    \[
    N(:,j) \leftarrow \epsilon \, K_W \, u \, N(:,j), 
    \]
    where $\epsilon \geq 0$ is the parameter controlling the noise level, and $u$ is drawn uniformly at random between 0 and 1 (hence the columns are perturbed with different noise levels which is more realistic). 

\end{itemize}

Finally, the negative entries of $M = [WH, \, Z] + N$ are set to zero (note that this can only reduce the noise). \\


Once an algorithm was run on a data set and has generated $r$ clusters $\mathcal{K}'_k$ ($1 \leq k \leq r$), 
its performance is evaluated using the following criterion
\[
\text{Accuracy} = \max_{P \in [1,2,\dots,r]} \frac{1}{n} \left( \sum_{k=1}^r  | \mathcal{K}_k \cap \mathcal{K}'_{P(k)} | \right)  \; \in \;  [0,1], 
\]
where $[1,2,\dots,r]$ is the set of permutations of $\{1,2,\dots,r\}$, and $\mathcal{K}_k$ are the true clusters. 
Note that if a data point does not belong to any cluster (such as an outlier), it does not affect the accuracy. 
In other words, the accuracy can be equal to 1 even in the presence of outliers (as long as all other data points are properly clustered together).

\subsection{Results} 

For each noise level $\epsilon$ and each value of $s$ and $b$, we generate 25 
synthetic HSI's as described in Section~\ref{sd}. Figure~\ref{figsynt} reports the average accuracy; hence the higher the curve, the better. 
\begin{figure}[ht!] 
\centering
\begin{tabular}{cc}
\includegraphics[width=8.5cm]{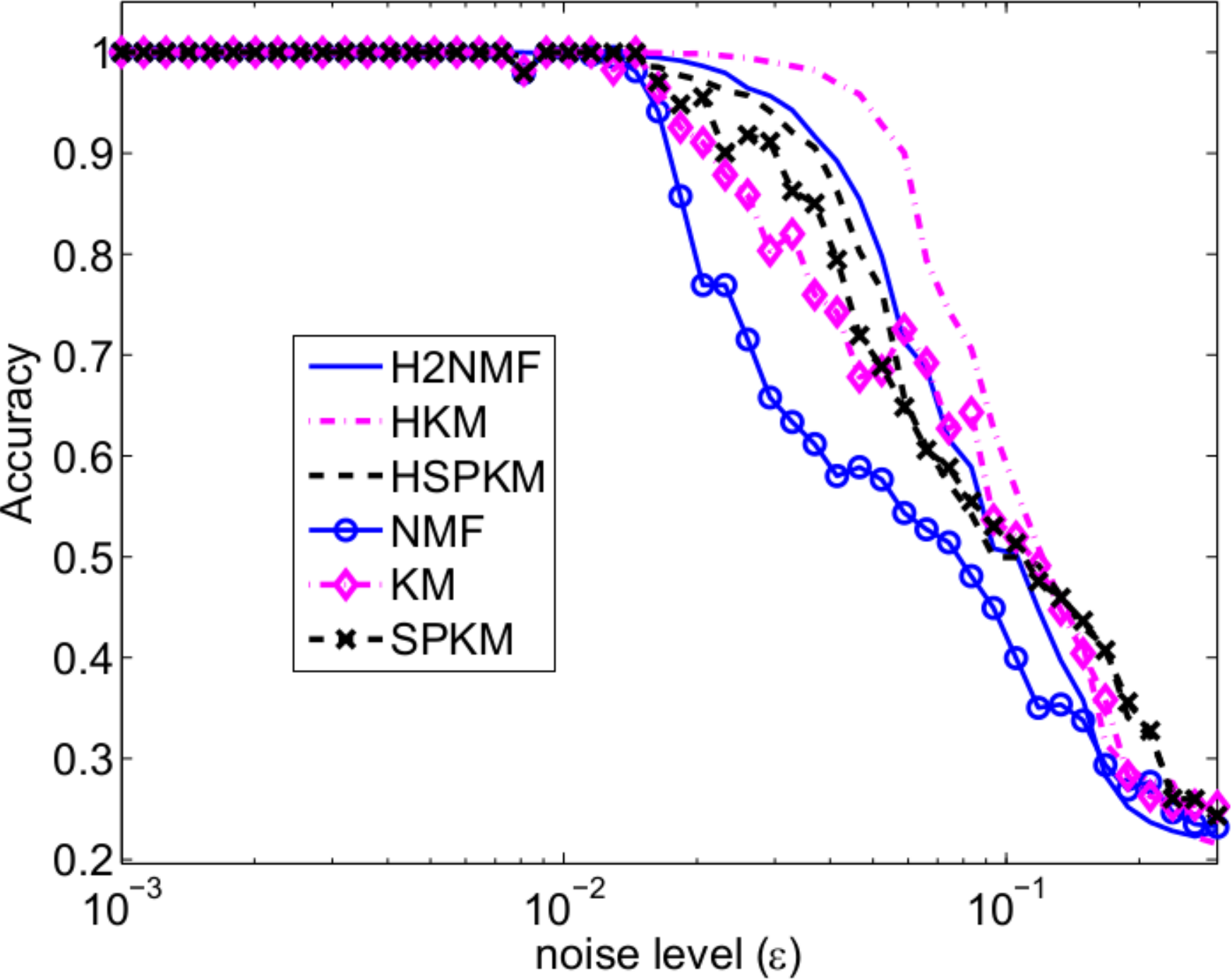} &
   \includegraphics[width=8.5cm]{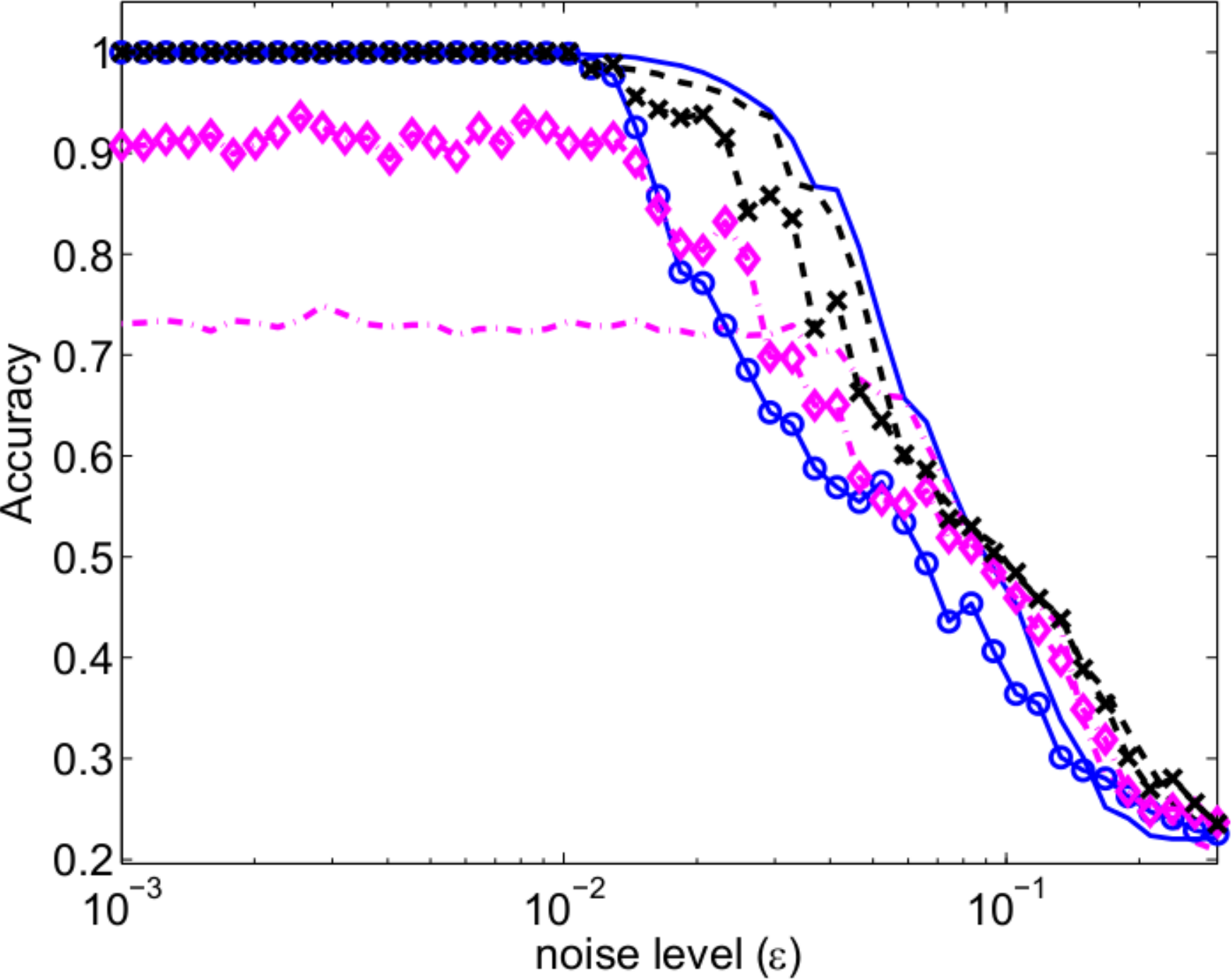}  \\
 \includegraphics[width=8.5cm]{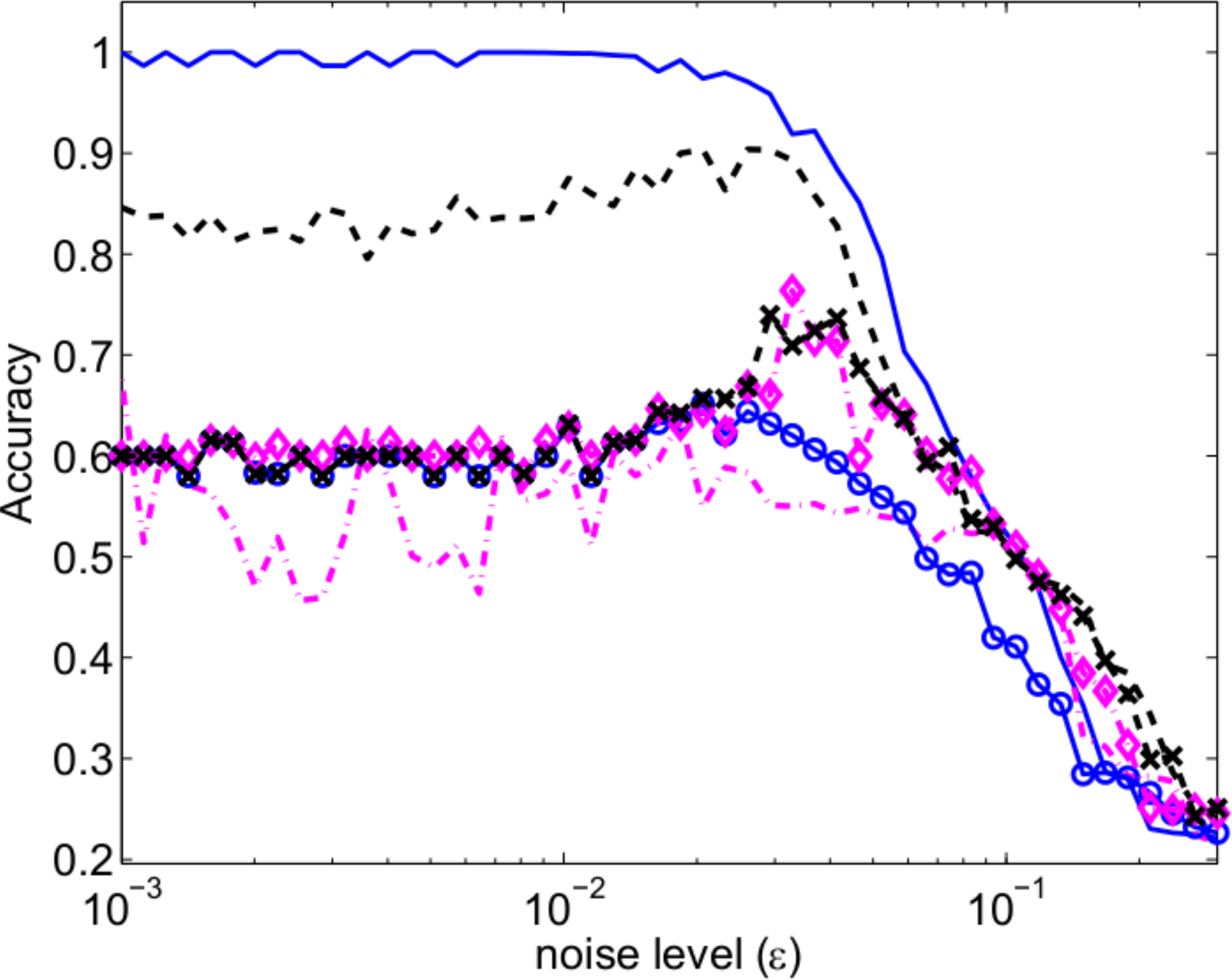}  & 
  \includegraphics[width=8.5cm]{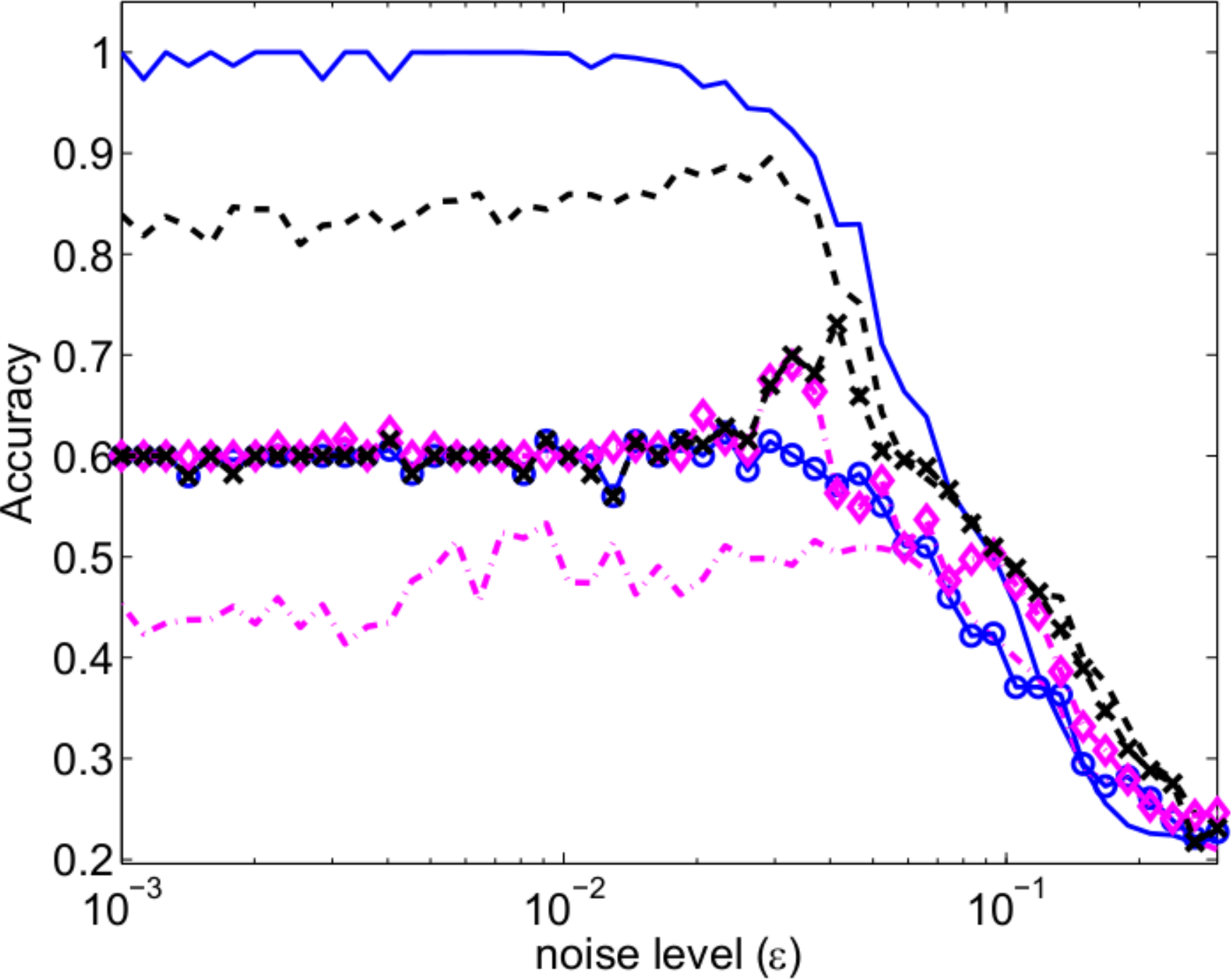} 
\end{tabular} 
\caption{Performance of the different algorithms on synthetic data sets. 
From top to bottom, left to right: $(s,b) = (0,0)$, (1,0), (0,1), and (1,1).}
\label{figsynt}
\end{figure} 

We observe that: 
\begin{itemize}

\item In almost all cases, the hierarchical clustering techniques consistently outperform the plain clustering approaches. 

\item Without scaling nor outliers (top left of Figure~\ref{figsynt}), 
HKM performs the best, while H2NMF is second best. 

\item With scaling but without outliers  (top right of Figure~\ref{figsynt}), H2NMF performs the best, slightly better than SPKM while HKM performs rather poorly. This shows that HKM is sensitive to scaling (that is, to different illumination conditions in the image), which will be confirmed on the  real-world HSI's.   

\item With outliers but without scaling  (bottom left of Figure~\ref{figsynt}), H2NMF outperforms all other algorithms. In particular, H2NMF has more than 95\% average accuracy for all $\epsilon \leq 0.3$. HSPKM behaves better than other algorithms but is not able to perfectly cluster the pixels, even for very small noise levels. 

\item With scaling and outliers (bottom right of Figure~\ref{figsynt}), HKM performs even worse. H2NMF still outperforms all other algorithms, while HSPKM extracts relatively good clusters compared to the other approaches. 

\end{itemize}

Table~\ref{timesynt} gives the average computational time (in seconds) of all algorithms for clustering a single synthetic data set. We observe that SPKM is significantly faster than all other algorithms while HKM is slightly slower. 
\begin{table}[ht!]
\begin{center}
\begin{tabular}{|c|c|c|c|c|c|}
\hline
 H2NMF &  HKM  & HSPKM & NMF    & KM   & SPKM  \\ \hline 
 1.68  &  2.77 &  1.78 &  2.25  & 3.73 &  0.19 \\
\hline
\end{tabular}
\caption{Average running time in seconds for the different algorithms on the synthetic data sets.}
\label{timesynt}
\end{center}
\end{table} 



\subsection{Real-World Hyperspectral Images}  \label{realworld}

In this section, we show that H2NMF is able to perform very good clustering of high resolution real-world HSI's. 
This section will focus on illustrating two important contributions: 
(i) H2NMF performs better than standard clustering techniques on real-world HSI, 
(ii) although H2NMF has been design to deal with HSI's with pixels dominated mostly by one endmember, 
it can provide meaningful and useful results in more difficult settings, and 
(iii)  H2NMF can be used as an endmember extraction algorithm in the presence of pure pixels (we compare it to vertex component analysis (VCA)~\cite{ND05} and the successive projection algorithm (SPA)~\cite{MC01}). 
Note that because the ground truth of these HSI's is not known precisely, 
it is difficult to provide an objective quantitative measure for the cluster quality.

\subsubsection{H2NMF as an Endmember Extraction Algorithm} \label{eea}

Once a set of clusters $\mathcal{K}_k$ ($1 \leq k \leq r$) has been identified by H2NMF (or any other clustering technique), each cluster of pixels should roughly correspond to a single material hence $M(:,\mathcal{K}_k)$ ($1 \leq k \leq r$) should be close to rank-one matrices. Therefore, as explained in Section~\ref{choice}, it makes sense to approximate these matrices with their best-rank one approximation: For $1 \leq k \leq r$, 
\[
M(:,\mathcal{K}_k) \approx u_k v_k^T, \quad \text{ where } u_k \in \mathbb{R}^m, v_k \in \mathbb{R}^n. 
\]
Note that, by the Perron-Frobenius and Eckart-Young theorems, $u_k$ and $v_k$ ($1 \leq k \leq r$) can be taken nonnegative since $M$ is nonnegative. 
Finally, $u_k$ should be close (up to a scaling factor) to the spectral signature of the endmember corresponding to the $k$th cluster. To extract a (good) pure pixel, a simple strategy is therefore to extract a pixel in each $\mathcal{K}_k$  whose spectral signature is the closest, with respect to some measure, to $u_k$. 
In this paper, we use the mean-removed spectral angle (MRSA) between $u_k$ and the pixels present in the corresponding cluster (see, e.g.,~\cite{cup11}). Given two spectral signatures, $x, y \in \mathbb{R}^m$, it is defined as 
\begin{equation} \label{mrsa}
\phi(x,y) 
= \frac{1}{\pi}
\arccos \left( \frac{ (x-\bar{x})^T (y-\bar{y}) }{||x-\bar{x}||_2 ||y-\bar{y}||_2} \right) \quad \in \quad [0,1],  
\end{equation}
where, for a vector $z \in \mathbb{R}^m$,  $\bar{z} = \left(\sum_{i=1}^m z_i\right) e$ and $e$ is the vector of all ones.

As we will see, this approach is rather effective for high-resolution images, and much more robust to noise and outliers than VCA and SPA. This will be illustrated later in this section. (It is important to keep in mind that SPA and VCA require the pure-pixel assumption while H2NMF requires that most pixels are dominated mostly by one endmember.)

\subsubsection{Urban HSI}

The Urban HSI\footnote{Available at \url{http://www.agc.army.mil/}.} from the HYper-spectral Digital Imagery Collection Experiment (HYDICE) contains 162 clean spectral bands, and the data cube has dimension $307 \times 307 \times 162$. The Urban data set is a rather simple and well understood data set: it is mainly composed of 6 types of materials (road, dirt, trees, roof, grass and metal) as reported in~\cite{GWO09}; see Figure~\ref{urband} 
and Figure~\ref{i2hnmf}. 
\begin{figure}[ht!] 
\begin{center}
\begin{tabular}{cc}
\includegraphics[height=6cm]{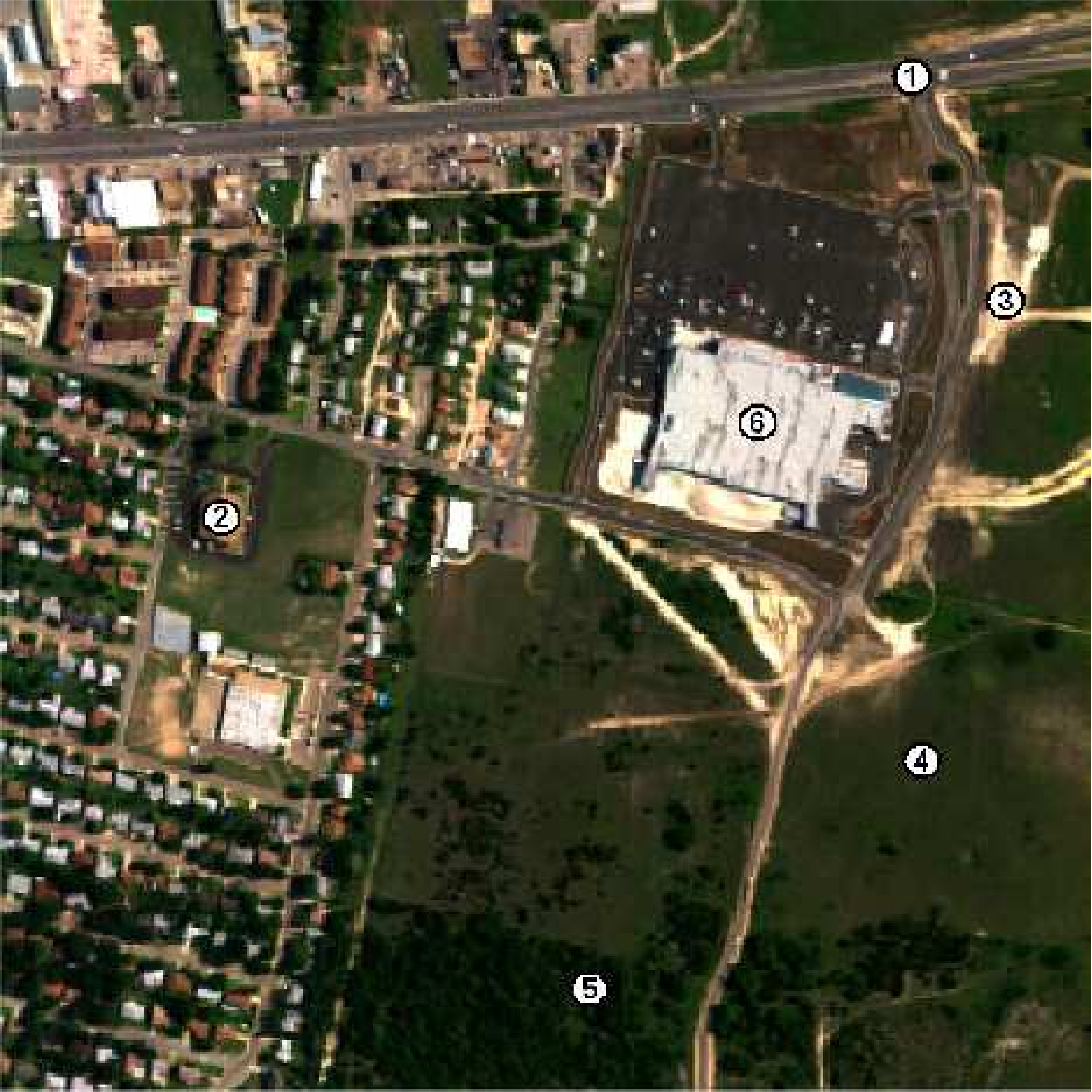} & 
 \includegraphics[height=6cm]{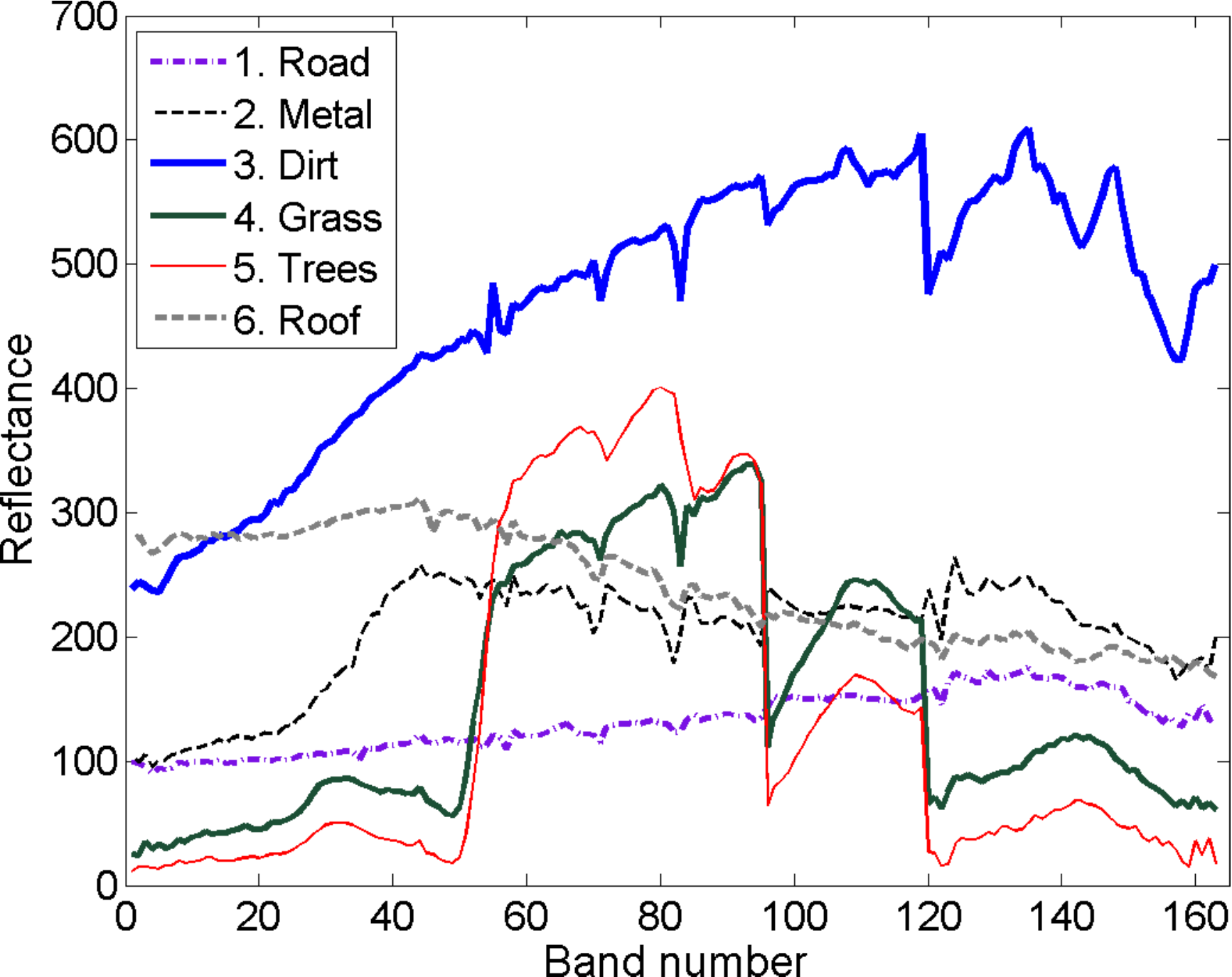} \\
\end{tabular}
\caption{Urban HSI set taken from an aircraft (army geospatial center). The spectral signatures of the six endmembers on the right-hand side were obtained using the N-FINDR5 algorithm~\cite{Win99} plus manual adjustment~\cite{GWO09}.}
\label{urband}
\end{center}
\end{figure} 
Figure~\ref{urbanclus} displays the clusters obtained with H2NMF, HKM and HSPKM\footnote{The clustering obtained with KM and SPKM can be found in~\cite{PGAG11}; the clustering obtained with KM is rather poor while the one obtained with SPKM is similar to the one obtained with HSPKM.}. 
We observe that 
\begin{itemize}

\item HKM performs very poorly. This is due to the illumination which is uneven among the pixels in the image (which is very damaging for HKM as shown on the synthetic data sets in Section~\ref{sd}). 

\item HSPKM properly extracts the trees and roof, but the grass is extracted as three separate clusters, while the road, metal and dirt form a unique cluster. 

\item H2NMF properly extracts the trees, roof and dirt, while the grass is extracted as two separate clusters, and the metal and road form a unique cluster. 

\end{itemize}
The reason why H2NMF separates the grass before separating the road and metal is threefold: (i)~the grass is the largest cluster and actually contains two subclasses with slightly different spectral signatures (as reported in~\cite{GP11}; see also Figure~\ref{ssurban}), (ii)~the metal is a very small cluster, and (iii) the spectral signature of the road and metal are not so different (see Figure~\ref{urband}). Therefore, splitting the cluster containing the road and metal does not reduce the error as much as splitting the cluster containing the grass. It is important to note that our criterion used to choose the cluster to split at each step favors larger clusters as the singular values of a matrix tends to be larger when the matrix contains more columns (see Section~\ref{choice}). 
Although it works well in many situations (in particular, when clusters are relatively well balanced), other criterion might be preferable in some cases; this is a topic for further research.  

Figure~\ref{clushier} displays the first levels of the cluster hierarchy generated by H2NMF. We see that if we were to split the cluster containing the road and metal, they would be properly separated. 
Therefore, we have also implemented an interactive version of H2NMF (denoted I-H2NMF) where, at each step, the cluster to split is visually selected\footnote{This is also available at \url{https://sites.google.com/site/nicolasgillis/}. The user can interactively choose which cluster to split, when to stop the recursion and, if necessary, which clusters to fuse.}.  
Hence, selecting the right clusters to split (namely, splitting the road and metal, and not splitting the grass into two clusters) allows us to identifying all materials separately, see Figure~\ref{i2hnmf}. (Note that this is not possible with HKM and HSPKM.) 
\begin{figure*}[ht!]
\begin{center}
\begin{tabular}{c}
\includegraphics[width=\textwidth]{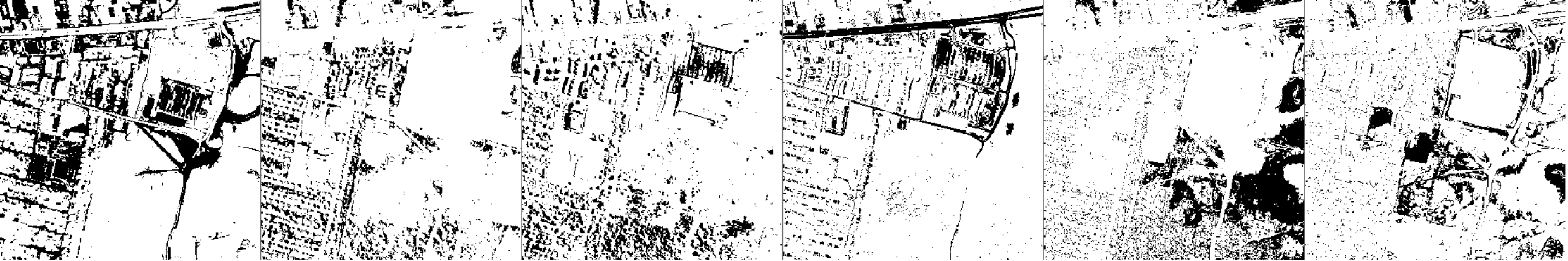} \vspace{0.2cm}\\ 
 \includegraphics[width=\textwidth]{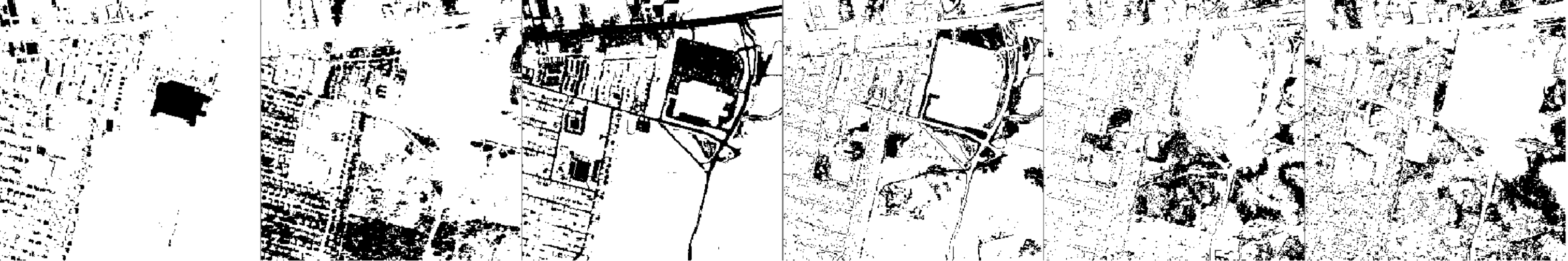} \vspace{0.2cm} \\
  \includegraphics[width=\textwidth]{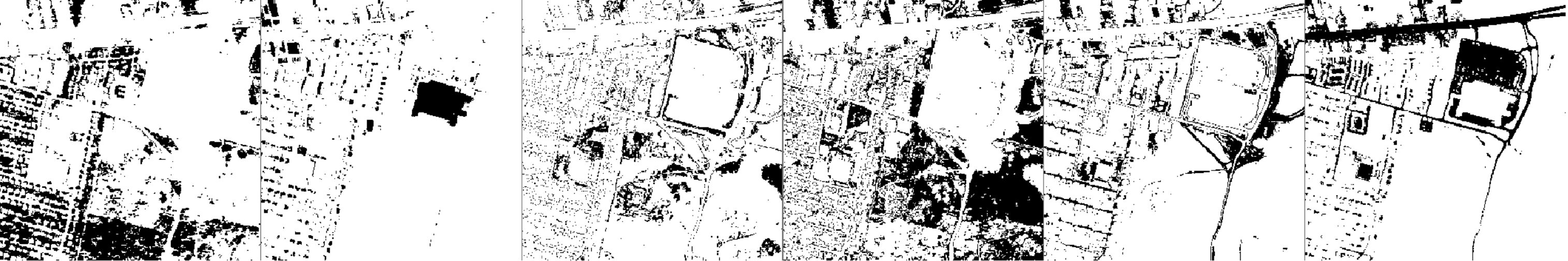} 
\end{tabular}
\caption{Clustering of the Urban HSI. From top to bottom: HKM, HSPKM, and H2NMF.}
\label{urbanclus} 
\end{center}
\end{figure*} 

\begin{figure*}[ht!]
\begin{center}
\includegraphics[width=\textwidth]{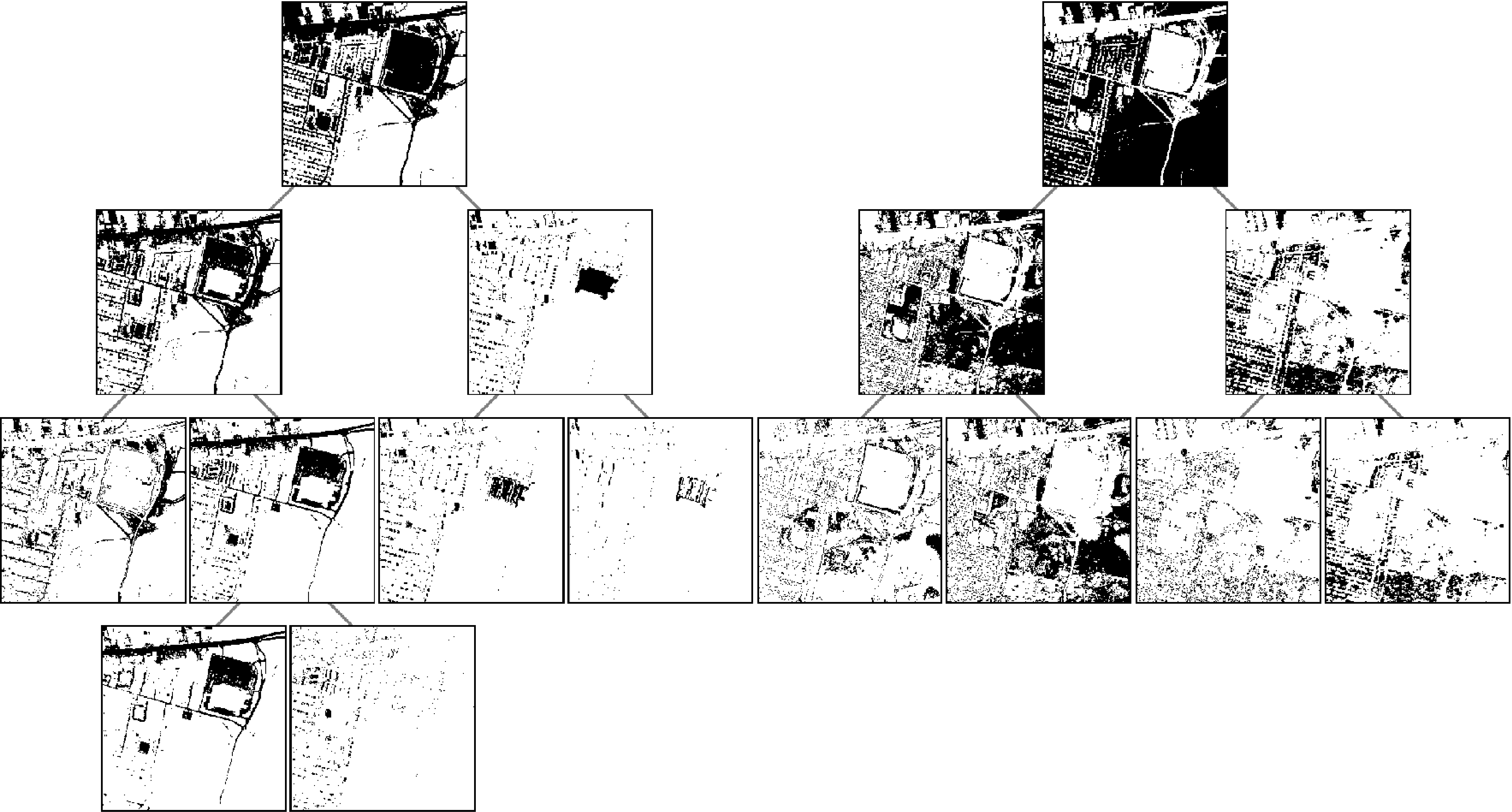}
\caption{Hierarchical structure of H2NMF for the Urban HSI.}
\label{clushier}
\end{center}
\end{figure*}

\begin{figure*}[ht!]
\begin{center}
\includegraphics[width=\textwidth]{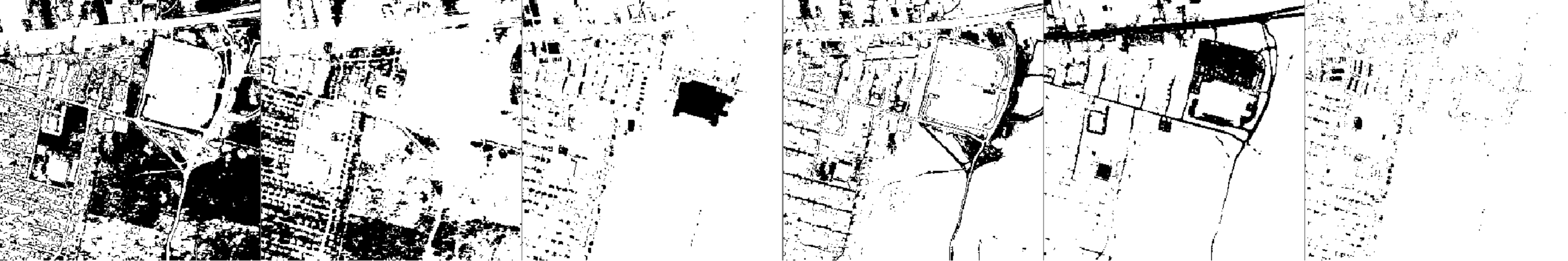}
\caption{Interactive H2NMF (I-H2NMF) of the Urban HSI (see also Figure~\ref{clushier}). From left to right: grass, trees, roof, dirt, road, and metal.}
\label{i2hnmf}
\end{center}
\end{figure*}

Using the strategy described in Section~\ref{eea}, we now compare the different algorithms when they are used for endmember extraction. Figure~\ref{ssurban} displays the spectral signatures of the pixels extracted by the different algorithms. 
\begin{figure*}[ht!]
\begin{center}
\includegraphics[width=\textwidth]{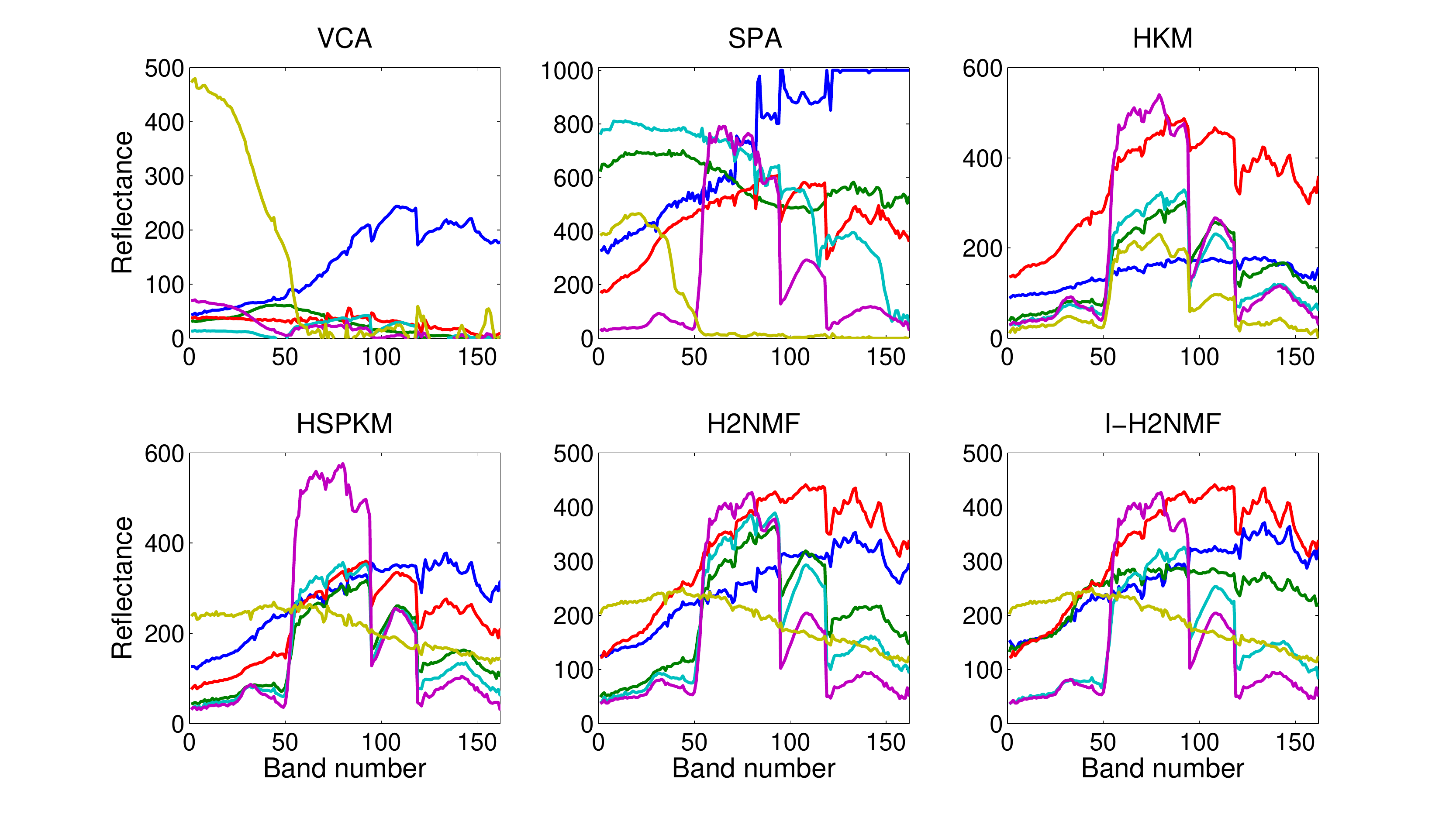}
\caption{Spectral signatures extracted by the different algorithms for the Urban HSI.}
\label{ssurban}
\end{center}
\end{figure*}
Letting $w'_k$ ($1 \leq k \leq r$) be the spectral signatures extracted by an algorithm, we match them with the `true' spectral signatures $w_k$ ($1 \leq k \leq r$) from~\cite{GWO09} so that $\sum_{k=1}^r  \phi(w_k, w'_k)$ is minimized; see Equation~\eqref{mrsa}.   
Table~\ref{timeaccurb} reports the MRSA, along with the running time of all methods. Although the hierarchical clustering methods are computationally more expensive, they perform much better than both VCA and SPA. 
\begin{table}[ht!]
\begin{center}
\begin{tabular}{|c|c|c|c|c|c|c|}
\hline
 & VCA & SPA  &  HKM &  \hspace{-0.1cm} HSPKM   \hspace{-0.1cm} & H2NMF &  \hspace{-0.2cm}  I-H2NMF  \hspace{-0.2cm}     \\ \hline 
Time (s.) & 3.62  &   1.10 &  113.79 &  47.30  & 41.00 &  43.18 \\ \hline  \hline   
Road & 13.09&11.54&14.97&11.54& 7.62& \textbf{7.27} \\
Metal & 51.53&62.31&30.72&31.42&28.61& \textbf{12.74} \\
Dirt & 60.81&17.35&10.97&13.56& \textbf{5.08}& \textbf{5.08}\\
Grass &16.65&47.32& \textbf{2.46}& 2.87& 3.39& 5.36 \\
Trees &53.38& 4.21& 2.16& 1.86& \textbf{1.63}& \textbf{1.63} \\
Roof &26.44&27.40&45.68& 8.84& \textbf{7.30} & \textbf{7.30} \\ \hline  
Average & 36.98&28.36&17.82&11.68& 8.94& \textbf{6.56} \\  \hline   
\end{tabular}
\caption{Running times and MRSA (in percent) for the Urban HSI.} 
\label{timeaccurb}
\end{center}
\end{table}


\subsubsection{San Diego Airport}
 

The San Diego airport is a HYDICE HSI containing 158 clean bands, and $400 \times 400$ pixels for each spectral image (i.e., $M \in \mathbb{R}^{158  \times 160000}_+$), see Figure~\ref{sandiego}. 
\begin{figure}[ht]
\begin{center}
\begin{tabular}{cc}
\includegraphics[height=6cm]{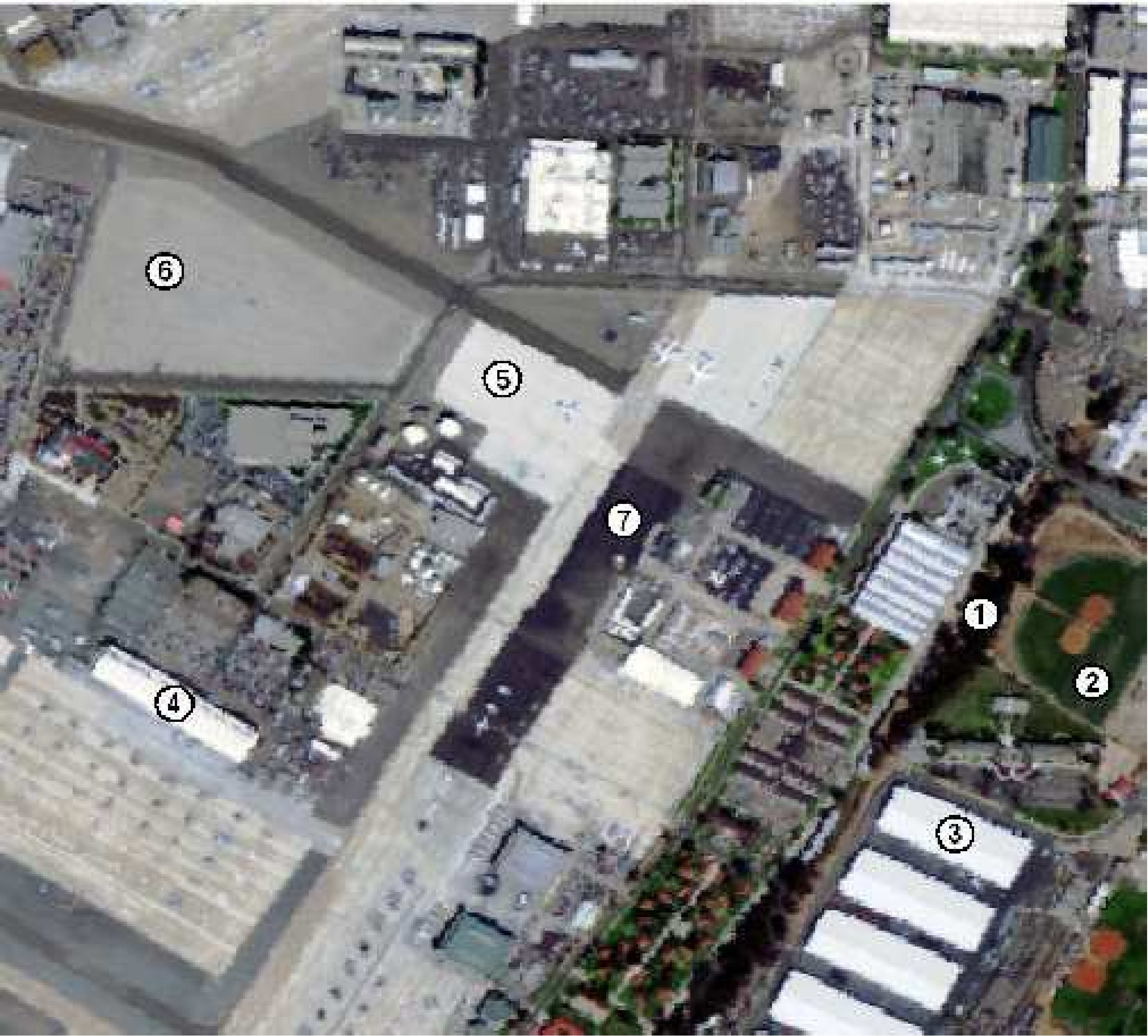} \vspace{0.1cm} & 
 \includegraphics[height=6cm]{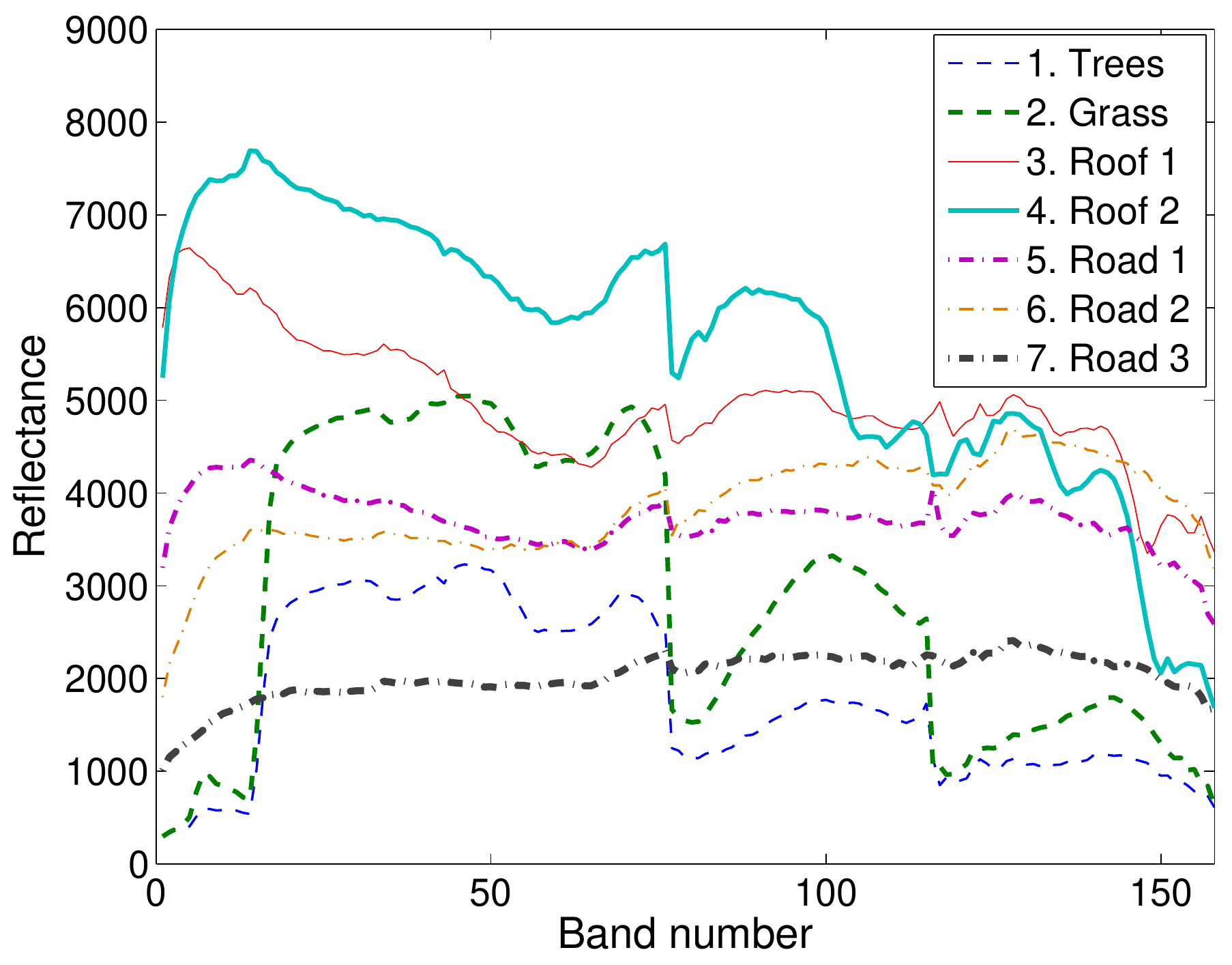} 
\end{tabular}
\caption{San Diego airport HSI (left) and manually selected spectral signatures (right).}
\label{sandiego}
\end{center}
\end{figure}
There are mainly four types of materials: road surfaces, roof, trees and grass; see, e.g.,~\cite{GP13}.  
There are three types of road surfaces including boarding and landing zones, parking lots and streets, and two types of roof tops\footnote{Note that in~\cite{GP13}, only one type of roof top is identified.}. 
In this section, we perform exactly the same experiment as in the previous section for the Urban HSI. The spectral signatures of the endmembers are shown on Figure~\ref{sandiego} and have been extracted manually using the HYPERACTIVE toolkit~\cite{FH07}.

Figure~\ref{sdclus} displays the clusters obtained with H2NMF, HKM and HSPKM. We observe that 
\begin{itemize}

\item HKM performs rather poorly. It cannot identify any roof tops, and four clusters only contain the vegetation. The reason is that the spectral signatures of the pixels containing grass and trees differ by scaling factors. 

\item HSPKM properly extracts the grass, two of the three road surfaces and the roof tops (although roof 2 is mixed with some road surfaces).  

\item H2NMF extracts all materials very effectively, except for the trees. 

\end{itemize}
The reason why H2NMF and HSPKM do not identify the trees properly is because the spectral signature of the trees and grass are almost the same: they only differ by a scaling factor (see Figure~\ref{sandiego}). Hence, for this data set, it is more effective to split the cluster containing the vegetation with $k$-means (which is the only one able to identify the trees, see the sixth abundance map of the first row of Figure~\ref{sdclus}). 
It would therefore be interesting to combine different methods for the splitting strategy; this is a topic for further research. \\

\begin{figure*}[ht!]
\begin{center}
\begin{tabular}{c}
\includegraphics[width=\textwidth]{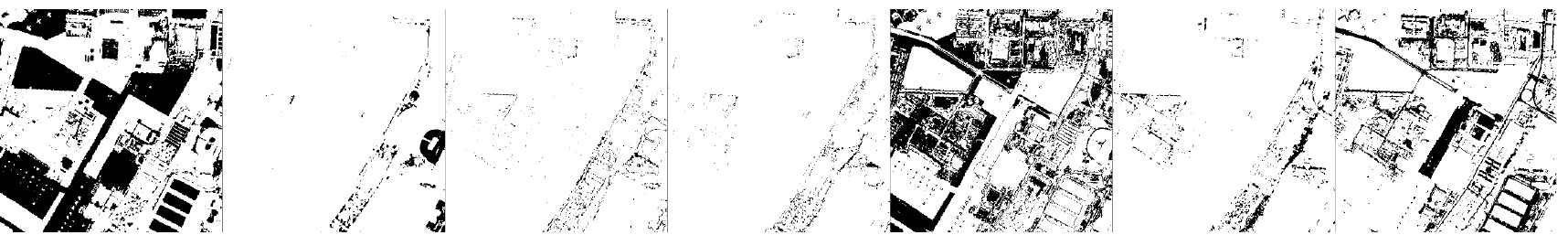}
 \vspace{0.2cm}\\ 
 \includegraphics[width=\textwidth]{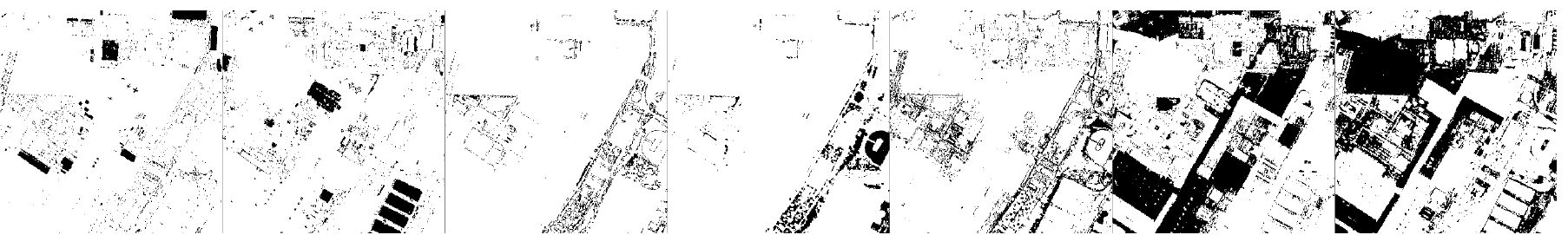} 
\vspace{0.2cm} \\
  \includegraphics[width=\textwidth]{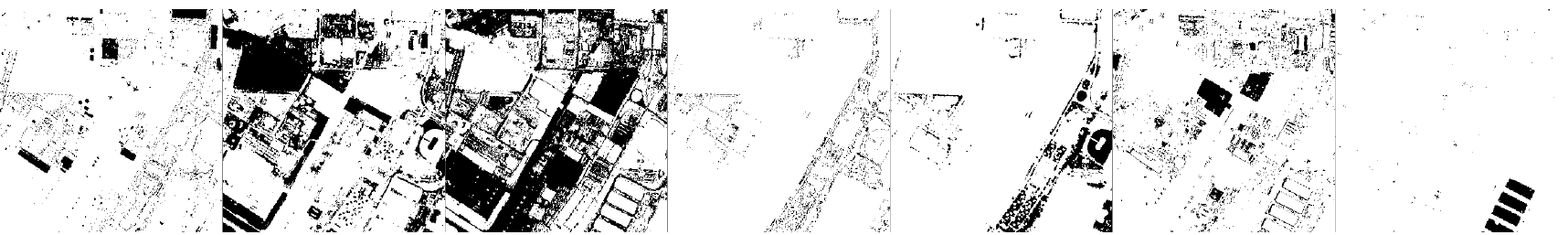} 
\end{tabular}
\caption{Clustering of the San Diego airport HSI. From top to bottom: HKM, HSPKM, and H2NMF.}
\label{sdclus} 
\end{center}
\end{figure*}

Figure~\ref{sssd} displays the spectral signatures of the pixels extracted by the different algorithms. 
\begin{figure*}[ht!]
\begin{center}
\includegraphics[width=16cm]{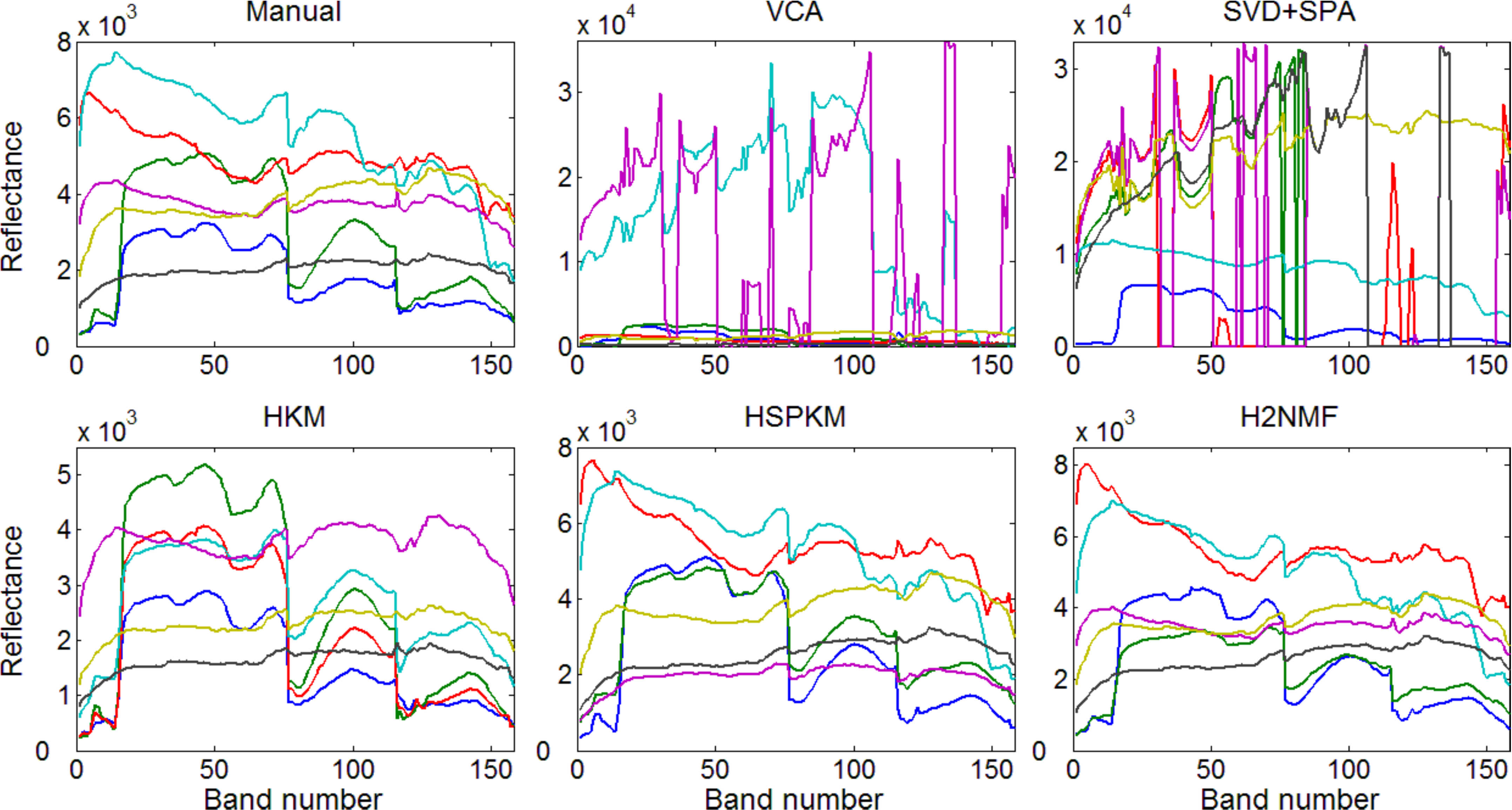}
\caption{Spectral signatures extracted by the different algorithms for the San Diego airport HSI.}
\label{sssd}
\end{center}
\end{figure*}
Table~\ref{timeaccsd} reports the running time of all methods, and the MRSA where the ground `truth' are the manually selected spectral signatures. 
The hierarchical clustering methods perform much better than VCA and SPA. In fact, as it can be observed on Figure~\ref{sssd}, VCA and SPA are very sensitive to outliers and both extract four of them (the San Diego airport HSI contains several outliers with very large spectral signatures). 
\begin{table}[ht!]
\begin{center}
\begin{tabular}{|c|c|c|c|c|c|}
\hline
 & VCA & SVD+SPA  &  HKM & HSPKM & H2NMF        \\ \hline 
Running time (s.) & 5.13  &   1.95 &  145.75 &  98.09  & 68.91  \\ \hline  \hline   
Trees &21.52&9.20&\textbf{2.97}&3.57&{3.42} \\
Grass &8.18&32.36&\textbf{2.37}&2.79&7.09\\
Roof 1 &16.06&38.30&47.43&\textbf{3.55}&{4.29}\\
Roof 2 &27.36&3.01&35.21& \textbf{1.64}&2.45\\
Road 1 &41.11&42.78&29.91&51.19&\textbf{9.58}\\
Road 2 &19.28&21.13&13.84&5.24&\textbf{3.77}\\
Road 3 &46.15&48.76&\textbf{5.32}&9.25&7.85\\ \hline   
Average &25.67&27.93&19.58& 11.03&\textbf{5.49} \\ \hline    
\end{tabular}
\caption{Running times and MRSA (in percent) for the San Diego airport HSI.}
\label{timeaccsd}
\end{center}
\end{table} 

Figure~\ref{abmapallsd} displays the abundance maps\footnote{We used the solution of $\min_{H \geq 0} ||M-WH||_F^2$ obtained with the NNLS solver from~\cite{KP11}.} corresponding to the spectral signatures displayed in Figure~\ref{sssd}. 
\begin{figure*}[h!]
\begin{center}
\includegraphics[width=\textwidth]{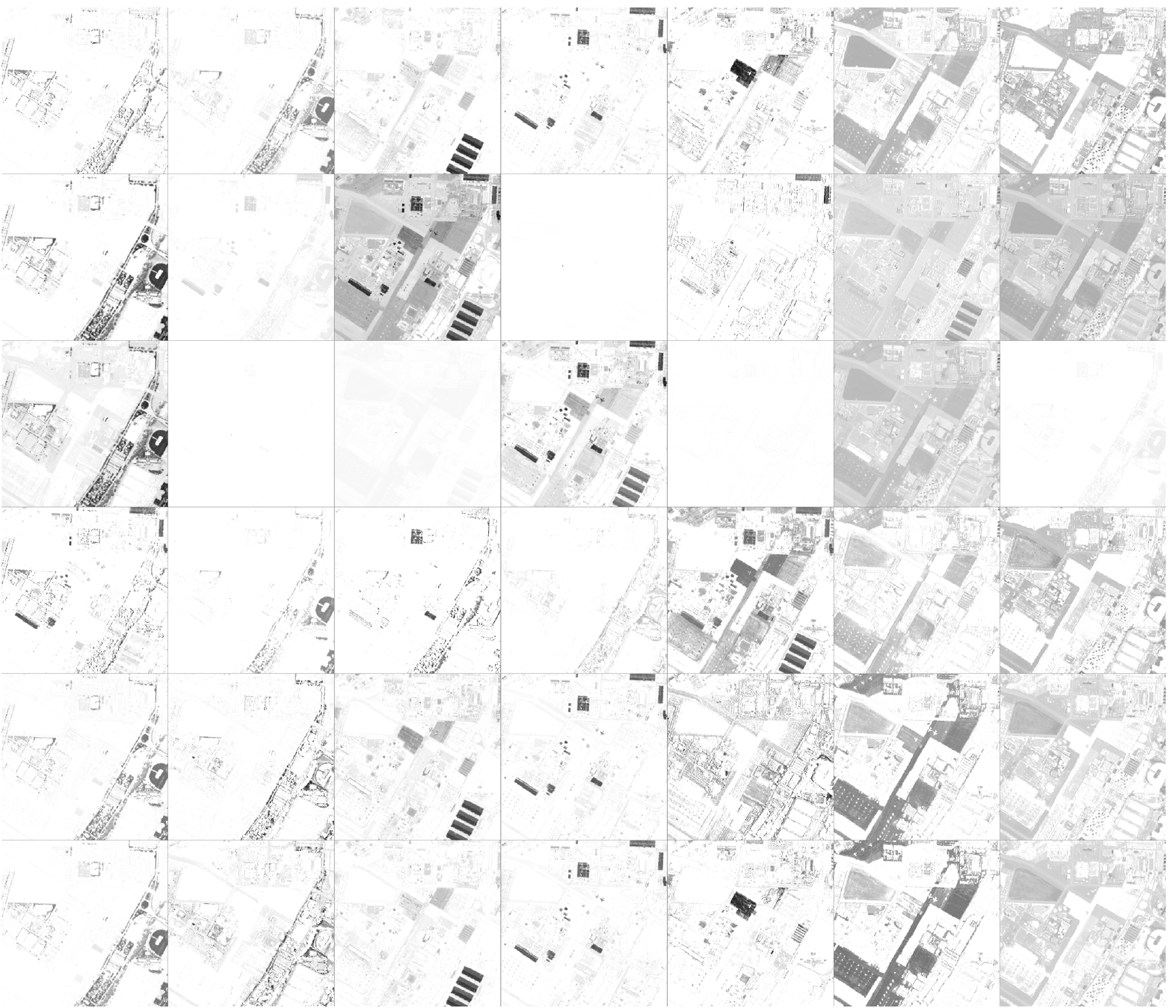} 
\caption{Comparison of the abundance maps obtained using the spectral signatures displayed in Figure~\ref{sssd}. 
From left to right: trees/grass, roof 1, roof 2, road 1, road 2, road 3. 
From top to bottom: Manual, VCA, SVD+SPA, HKM, HSPKM, H2NMF.} 
\label{abmapallsd}
\end{center}
\end{figure*} 
The abundance maps identified by H2NMF are the best: in fact, HSPKM does not identify road 1 which is mixed with roof 1. The reason is that the spectral signature extracted by HSPKM for road 1 is rather poor (see Table~\ref{timeaccsd}). Also, H2NMF actually identified better spectral signatures for roof 1 and road 1 than the manually selected ones for which the corresponding abundance maps on the first row of Figure~\ref{abmapallsd} contains more mixture. \\


Figure~\ref{clushiersd} displays the first levels of the cluster hierarchy of H2NMF. It is interesting to notice that road 2 and 3 can be further split up into two meaningful subclasses.  Moreover, another new material is identified (unknown to us prior to this study), it is some kind of roofing material/dirt (note that HKM and HSPKM are not able to identify this material). 
\begin{figure*}[h!]
\begin{center}
\includegraphics[width=\textwidth]{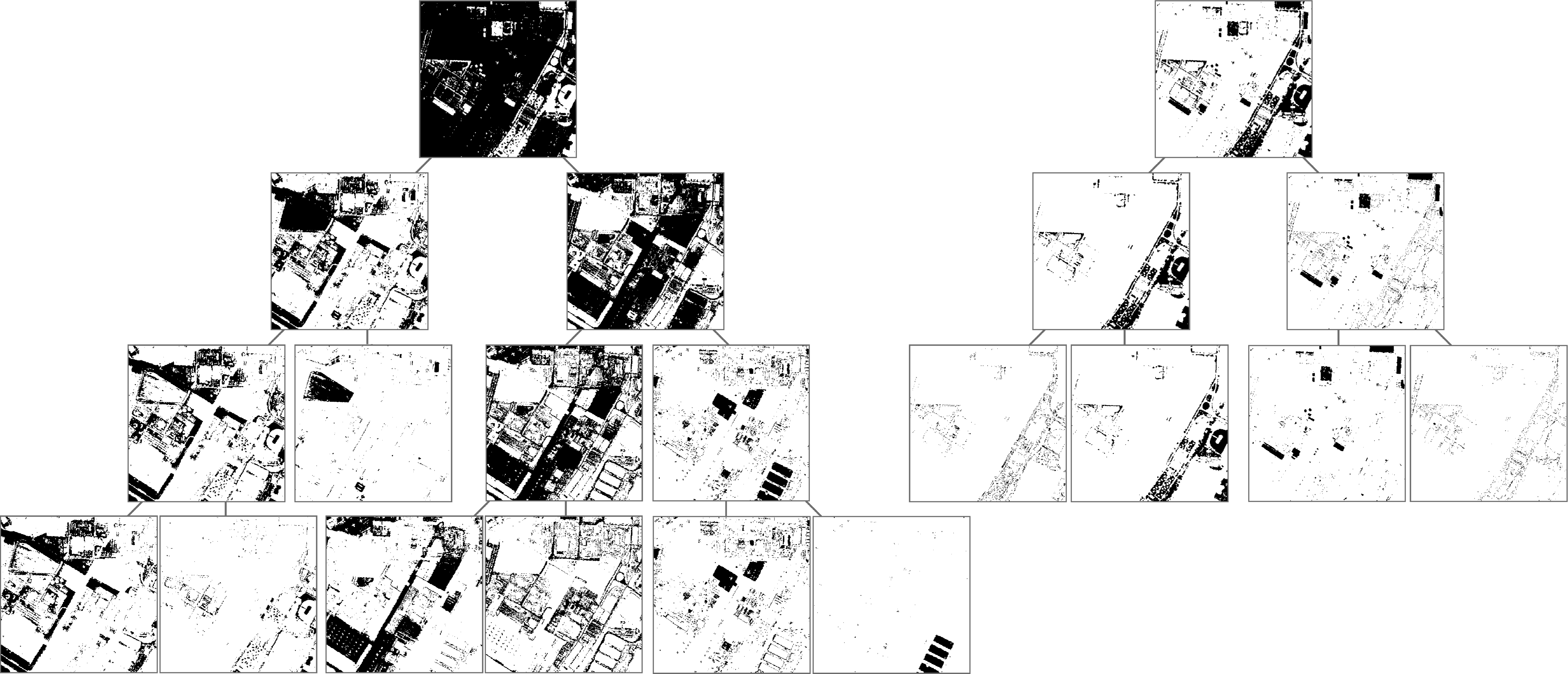}
\caption{Hierarchical structure of H2NMF for the San Diego airport HSI.}
\label{clushiersd}
\end{center}
\end{figure*}
Figure~\ref{abmapih2nmf} displays the clusters obtained using I-H2NMF, that is, manually splitting and fusing the clusters (more precisely, after having identified the new endmember by splitting road~2, we refuse the two clusters corresponding to road~2); see also Figure~\ref{clushiersd}. 
\begin{figure*}[h!]
\begin{center}
\includegraphics[width=\textwidth]{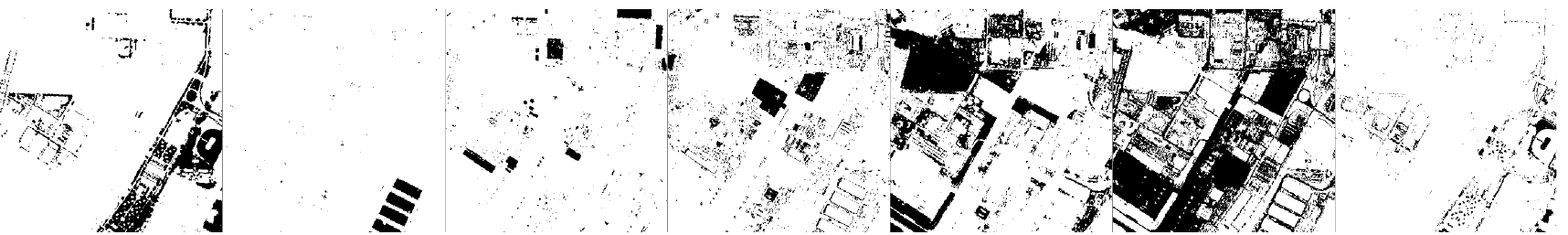} 
\caption{Clustering of San Diego airport HSI with I-H2NMF. 
From left to right, top to bottom: vegetation (grass and trees), roof~1, roof~2, road~1, road~2, road~3, dirt/roofing.} 
\label{abmapih2nmf}
\end{center}
\end{figure*}

\subsection{Additional experiments on real-world HSI's} 

In this section, our goal is not to compare the different clustering strategies (due to the space limitation) but rather show that H2NMF can give good results for other real-world and widely used data sets; in particular the Cuprite data set which is rather complicated with many endmembers and highly mixed pixels. 
We also take this opportunity to show that our Matlab code is rather easy to use and fast:

\subsubsection{Terrain HSI} 

The Terrain hyperspectral image is available from \url{ttp://www.agc.army.mil/Missions/Hypercube.aspx}. It is constituted of 166 cleans bands, each having $500 \times 307$ pixels, and is composed of about 5 different materials: road, tree, bare soil, thin and tick grass; see, e.g., \url{http://www.way2c.com/rs2.php}. The Matlab code to run H2NMF is the following:   
\small 
\begin{verbatim}
>> load Terrain; % Load HSI as matrix x
>> tic; [IDX, C] = hierclust2nmf(x,5); toc
Hierarchical clustering started... 
1...2...3...4...Done. 
Elapsed time is 20.261847 seconds. 
>> affclust(IDX,500,307,5); % Display the clusters; see Figure 16
>> figure; plot(C);  % Display endmembers; see Figure 17
>> H = nnlsm_blockpivot(C,x); % Compute abundance maps
>> affichage(H',5,500,307); % Display abundance maps; see Figure 18 
\end{verbatim}  
\normalsize 
\begin{figure}[ht!]
\begin{center}
\includegraphics[width=\textwidth]{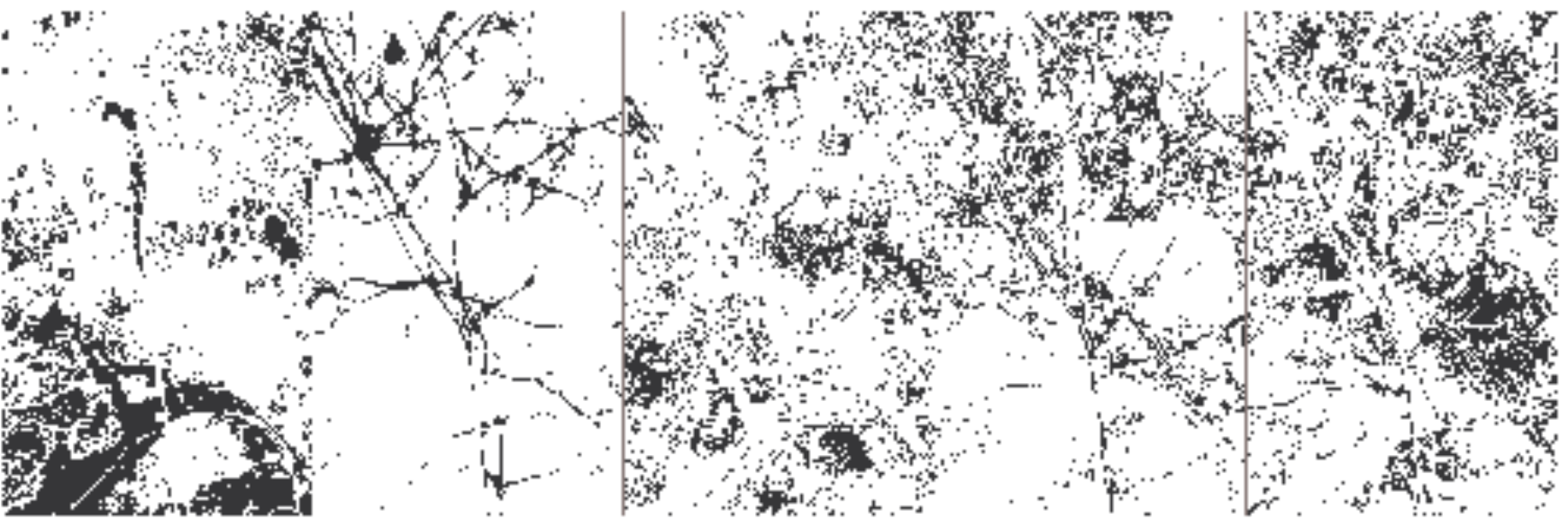}
\caption{Five clusters obtained automatically with H2NMF on the Terrain HSI. From left to right: tree, road, thick grass, bare soil and thin grass.} 
\label{terclus}
\end{center}
\end{figure}
\begin{figure}[ht!]
\begin{center}
\includegraphics[width=8cm]{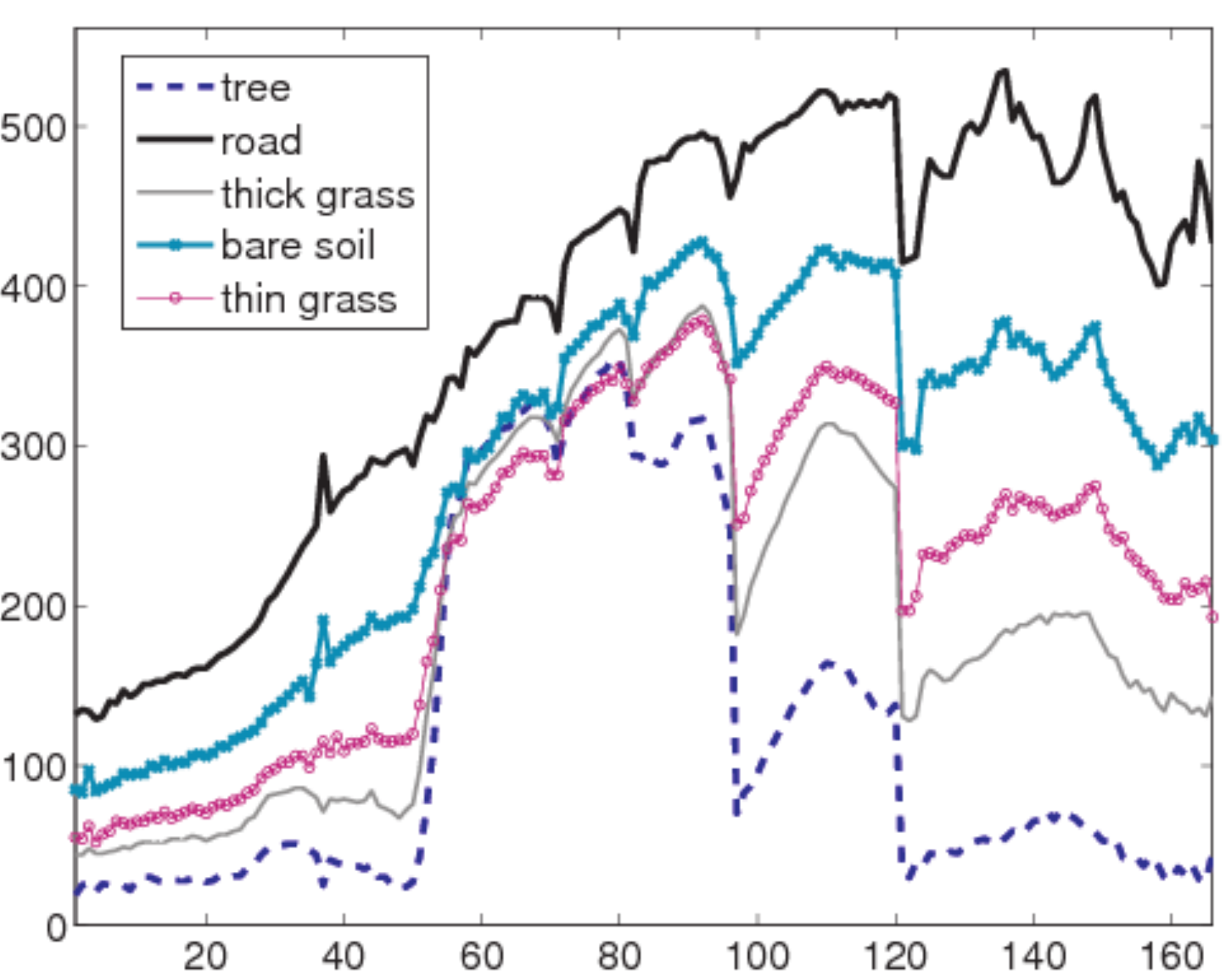}
\caption{Five endmembers obtained with H2NMF on the Terrain HSI.}
\label{terend}
\end{center}
\end{figure}
\begin{figure}[ht!]
\begin{center}
\includegraphics[width=\textwidth]{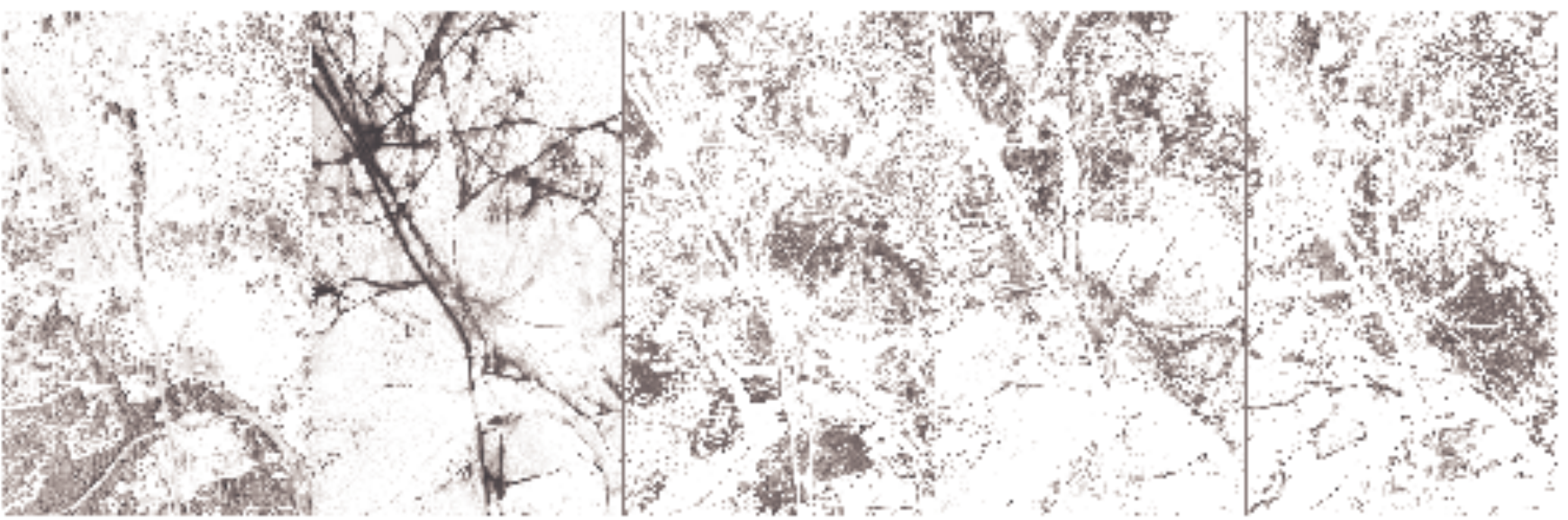}
\caption{Five abundance maps corresponding to the endmembers extracted with H2NMF.} 
\label{terab}
\end{center}
\end{figure} 
H2NMF is able to identify the five clusters extremely well, while HKM and HSPKM are not able to separate bare soil, thick and thin grass properly.

\subsubsection{Cuprite HSI} \label{cuprite} 

Cuprite is a mining area in southern Nevada with mostly mineral and very little vegetation, located approximately 200km northwest of Las Vegas, 
see, e.g.,~\cite{ND05, cup11} for more information and \url{http://speclab.cr.usgs.gov/PAPERS.imspec.evol/aviris.evolution.html}. 
It consists of 188 images, each having $250 \times 191$ pixels, and is composed of about 20 different minerals. 
The Cuprite HSI is rather noisy and many pixels are mixture of several endmembers.  
Hence this experiment illustrates the usefulness of H2NMF to analyze more difficult data sets, where the assumption that most pixels are dominated mostly by one endmember is only roughly satisfied; see Figure~\ref{cupH2NMF}. We run H2NMF with $r=15$: 
\small  
\begin{verbatim}
>> load cuprite_ref; %From www.lx.it.pt/~bioucas.  
>> tic; [IDX, C] = hierclust2nmf(x,15); toc
Hierarchical clustering started... 
1...2...3...4...5...6...7...8...9...10...
11...12...13...14...Done.
Elapsed time is 11.632038 seconds.
>> affclust(IDX,250,191,5); %See Figure 19 displaying the 15 clusters.
\end{verbatim}  
\normalsize 
\begin{figure}[ht!]
\begin{center}
\includegraphics[width=\textwidth]{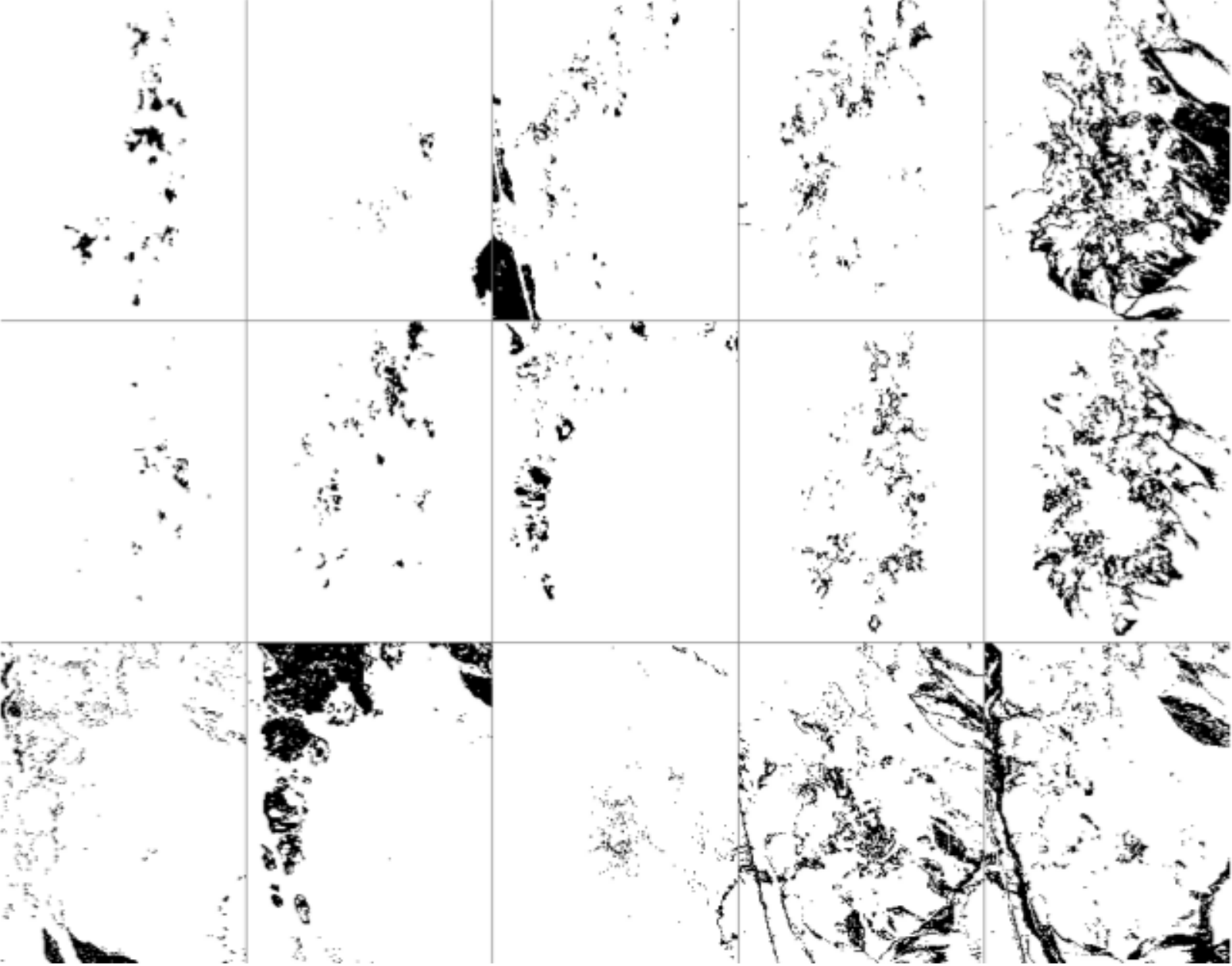}
\caption{Fifteen clusters obtained automatically with H2NMF on the Cuprite HSI. Some materials can be distinguished, e.g., (1)~Alunite, 
(2)~Montmorillonite, (3)~Goethite, (5)~Hematite, (8)-(12)~Desert Varnish, (11)~Iron oxydes, and (15)~Kaolinite (counting from left to right, top to botttom).}
\label{cupH2NMF}
\end{center}
\end{figure}

\section{Conclusion and Further Work}

In this paper, we have introduced a way to perform hierarchical clustering of high-resolution HSI's using the geometry of such images and the properties of rank-two NMF; see Algorithm 1 (referred to as H2NMF). 
We showed that the proposed method outperforms $k$-means, spherical $k$-means and standard NMF on several synthetic and real-world data sets, being more robust to noise and outliers, while being computationally very efficient, requiring $\mathcal{O}(mnr)$ operations 
($m$ is the number of spectral bands, $n$ the number of pixels and $r$ the number of clusters). 
Although high resolution HSI's usually have low noise levels, one of the reason H2NMF performs well is that it can handle better background pixels and outliers. There might also be some materials present in very small proportion that are usually modeled as noise~\cite{Jose12} 
(hence robustness to noise is a desirable property even for high resolution HSI's). Moreover, we also showed how to use H2NMF to identify pure pixels which outperforms standard endmember extraction algorithms such as VCA and SPA.

It would be particularly interesting to use other priors of HSI's to perform the clustering. In particular, using the spatial information (that is, the fact that neighboring pixels are more likely to contain the same materials) could certainly improve the clustering accuracy.  Also, the same technique could be applied to other kinds of data (e.g., in medical imaging, or document classification).

	 \section*{Acknowledgments}
 
The authors would like to thank the editor and the reviewers for their insightful comments which helped improve the paper. \vspace{0.2cm}

\bibliographystyle{spmpsci}  
\bibliography{rank2nmf}

\newpage

\end{document}